\documentclass{article}

\usepackage{natbib}        
\usepackage{amsmath}
\usepackage[final]{neurips_2025}
\usepackage{comment}
\usepackage{tcolorbox,float}
\newtheorem{assumption}{Assumption}
\newtheorem{theorem}{Theorem}
\newtheorem{definition}{Definition}
\newtheorem{corollary}{Corollary}
\newtheorem{lemma}{Lemma}
\usepackage{algorithm}
\usepackage{times}
\usepackage{algpseudocode}%
\DeclareUnicodeCharacter{202F}{\,}

\usepackage[utf8]{inputenc} 
\usepackage[T1]{fontenc}    
\usepackage{hyperref}       
\usepackage{url}            
\usepackage{booktabs}       
\usepackage{amsfonts}       
\usepackage{nicefrac}       
\usepackage{microtype}      
\usepackage{xcolor}         
\title{Scalable Policy-Based RL Algorithms for POMDPs}

\author{%
  Ameya Anjarlekar \\
  UIUC\\
  \texttt{ameyasa2@illinois.edu} \\
  \And
  S.~Rasoul~Etesami \\
  UIUC\\
  \texttt{etesami1@illinois.edu} \\
    \And
  R.~Srikant \\
  UIUC\\
  \texttt{rsrikant@illinois.edu} \\
}

\begin{document}

\maketitle

\begin{abstract}
The continuous nature of belief states in POMDPs presents significant computational challenges in learning the optimal policy. In this paper, we consider an approach that solves a Partially Observable Reinforcement Learning (PORL) problem by approximating the corresponding POMDP model into a finite-state Markov Decision Process (MDP) (called \textit{Superstate} MDP). We first derive theoretical guarantees that improve upon prior work that relate the optimal value function of the transformed Superstate MDP to the optimal value function of the original POMDP. Next, we propose a policy-based learning approach with linear function approximation to learn the optimal policy for the \textit{Superstate} MDP. Consequently, our approach shows that a POMDP can be approximately solved using TD-learning followed by Policy Optimization by treating it as an MDP, where the MDP state corresponds to a finite history. We show that the approximation error decreases exponentially with the length of this history. To the best of our knowledge, our finite-time bounds are the first to explicitly quantify the error introduced when applying standard TD learning to a setting where the true dynamics are not Markovian.
\end{abstract}

\section{Introduction}

Reinforcement learning (RL) provides a robust and systematic framework for solving sequential decision-making problems by modeling these problems as Markov Decision Processes (MDPs). However, many real-world systems such as robotic controllers must operate under uncertainty, handling incomplete, noisy, or ambiguous data, making traditional RL techniques ineffective for such problems \cite{cass2}. 

To address such challenges, partially observable reinforcement learning (PORL) extends the RL framework by modeling these problems as Partially Observable Markov Decision Processes (POMDPs). This approach accounts for hidden states, enabling effective decision-making under uncertainty. Beyond these applications, PORL is widely applied in areas such as autonomous driving \cite{Held-2011}, personalized content recommendations \cite{News}, medical diagnosis \cite{HAUSKRECH}, and games \cite{Brown2019}, where decision-making under partial observability is crucial. As a result, addressing decision-making under uncertainty has become a critical topic across diverse fields such as operations research and healthcare.



Although POMDPs provide a general framework for modeling decision-making under uncertainty, solving them presents significant computational challenges even when the exact model of the POMDP is known. The early works of \cite{smallwood, ASTROM} focused on the planning problem and effectively converted POMDPs into MDPs via belief states. However, the continuous nature of belief states and the PSPACE-completeness of solving POMDPs \cite{papadimitriou, vlassis} make these problems computationally intractable. 

Various approaches have been proposed to address these limitations, with the aim of providing approximate solutions for learning POMDPs by selecting actions based on a finite history of observations. For example, \cite{Jaakkola, Williams, aziz} select actions based solely on current observations, while others, such as \cite{loch, Littman}, consider past $k$ observations to guide decision making. Nevertheless, while these approaches demonstrate empirical success, they lack rigorous theoretical performance guarantees. 

In contrast, for learning fully observable MDPs, a wide range of Reinforcement Learning algorithms exist with provable sample complexity bounds, ranging from value-iteration-based approaches such as Q-learning to policy-iteration-based algorithms such as actor-critic and natural policy gradient (NPG). A comprehensive survey of these algorithms can be found in \cite{bertsekas} and \cite{sutton}. 

The empirical success of finite-history-based algorithms for POMDPs and the availability of performance guarantees for standard RL algorithms motivate a critical research question: \emph{Can we leverage standard RL algorithms to approximately learn an optimal policy for a POMDP by treating it as an MDP, where the states correspond to the finite histories? Specifically, can we establish theoretical performance bounds for such an approach?} 

In this paper, we address this question by considering a previously proposed approximation by \cite{kara} that maps a POMDP to a finite-state MDP, called the \textit{ super-state} MDP by restricting the information at each time instant to a finite history of past observations. However, this approximation is difficult to study for the following reasons:
\begin{itemize}
    \item It is unclear whether using standard TD learning as the policy evaluation step in a problem where the true model is a POMDP would result in good performance or even convergence.
    \item The TD learning algorithm used in \cite{semih} uses a $m$-step version of TD as a workaround, which is computationally more expensive than standard TD learning.
\end{itemize}
In view of the limitations of the prior work, our main contributions are stated as follows:

\begin{itemize}
\item\textbf{Approximation guarantees for POMDP to MDP transformation:} A standard approach to assess the quality of a transformed MDP is to bound the difference between its optimal value function and that of the original POMDP. In prior work such as \cite{kara, mahajan}, this difference is bounded by the square of the expected horizon length, i.e., $\frac{1}{(1-\gamma)^2}$ where $\gamma$ is the discount factor, while, \cite{abel} provides a bound of $\frac{1}{(1-\gamma)^3}$, and the bound in \cite{semih} is polynomial in the horizon length. We improve upon these results by providing a tighter bound. Additionally, while \cite{kara} provides bounds on the expected difference between the optimal value functions, we establish a stronger worst-case bound.

\item\textbf{A General Purpose Algebraic Identity}: A key reason for improved bounds is the introduction of a novel algebraic result, i.e., Lemma \ref{lemma: norm bound} which proves effective in bounding expressions of the form \( \big|\sum_{i=1}^m (a_i b_i - c_i d_i) \big| \), when vectors \( (\boldsymbol{a}, \boldsymbol{c}) \) and \( (\boldsymbol{b}, \boldsymbol{d}) \) are close to each other, and \( (\boldsymbol{b}, \boldsymbol{d}) \) corresponds to probability mass functions.  Traditional approaches rely on decompositions and triangle inequality, often leading to loose bounds. In contrast, our refined approach yields tighter guarantees, improving the existing bounds. 

Our algebraic result is of broader interest and can be used to improve bounds in other analyses. As an example, we refine\footnote{Proof is in the Supplementary.} the performance bounds in \cite{mahajan}, which introduced approximate information states but lacked a systematic method for constructing them. Additionally, while their deep learning-based approach provides an empirical solution, it optimizes an objective misaligned with the theoretical guarantees. 
    
     \item \textbf{Addressing Challenges in Policy Optimization:} Having established that the Superstate MDP is a strong approximation of the POMDP, we investigate whether standard Policy Optimization algorithms can effectively learn its optimal policy. A key challenge arises from the fact that learning samples correspond to belief states of the original POMDP rather than the Superstate MDP, creating a sampling mismatch that raises concerns about whether applying standard RL techniques can yield strong performance guarantees.

\cite{semih} addresses this issue by modifying the TD-learning part of Policy Optimization by employing an $m$-step TD-learning approach, leading to significantly higher computational complexity. In contrast, our work is the first to establish convergence guarantees for a standard policy optimization algorithm that alternates between learning the Q-function for a fixed policy using TD-learning and updating the policy accordingly. This avoids additional computational overhead while providing stronger theoretical guarantees.
     
    \item \textbf{Performance Bounds for Linear Function Approximation Setting:} Apart from  \cite{semih} which leads to a higher computational cost, \cite{kara} considers Q-learning to learn the optimal policy for the Superstate MDP. However, along with relying on ergodicity assumptions, their approach faces significant scalability challenges due to the difficulty of proving convergence for Q-learning with linear function approximation. In contrast, we develop performance bounds for the Policy Optimization Algorithm and extend our analysis to the function approximation case, enabling scalability to large state spaces. 

A key challenge in proving convergence for Policy Optimization algorithms lies in the non-stationary nature of the sampling policy during the TD-learning phase. In contrast, the Q-learning approach in \cite{kara} can leverage a stationary exploratory sampling policy, simplifying the analysis. Additionally, they assume an additional ergodicity condition and provide only asymptotic convergence guarantees. In comparison, our approach avoids such restrictive assumptions and establishes finite-time performance bounds. 

 Finally, another minor contribution of our work is to extend the analysis for the POLITEX algorithm in \cite{politex} to the discounted reward setting and analyze the regret of our algorithm with respect to the optimal value function of the POMDP. 
    
\end{itemize}


\section{Related Work}

\medskip
\noindent\textbf{Planning algorithms for POMDP:} Early POMDP algorithms assumed full model knowledge and focused on exact belief-state calculations. Techniques such as Witness algorithm \cite{Cassandra}, Lovejoy's suboptimal algorithm \cite{lovejoy}, and incremental pruning \cite{Zhang} were aimed at making planning more efficient. Comprehensive surveys can be found in \cite{Krishnamurthy_2016, Murphy}. However, these approaches suffer from exponential complexity growth due to reliance on precise belief-state calculations, limiting their practicality. 

\medskip
\noindent\textbf{Using internal state representations to solve PORL problems:} To address the challenge of continuous belief states, internal state representations have been proposed. Works like \cite{Jaakkola, Williams, aziz} select actions based solely on current observations, while \cite{loch, Littman}, consider past $k$ observations to guide decision-making. Additionally, sliding window controllers \cite{Williams, loch, kara, sung, yu, cayci2024RNN, amiri} and memory-based techniques \cite{McCallum1993, meuleau} have demonstrated practical success. However, most of these methods are heuristic and lack formal guarantees.

\medskip
\noindent\textbf{Learning algorithms with provable guarantees:}
Recent efforts have focused on developing algorithms for solving PORL problems with theoretical guarantees. \cite{simon, Efroni2022} provide provable bounds, though these rely on the state decodability assumption, which implies that a finite history of observations can perfectly infer the current state. Other approaches \cite{lingwang, liu2022, Jin2020} target specific subclasses of PORL problems.

\medskip
\noindent\textbf{Deep learning inspired techniques:} It is also worth mentioning prior work on empirical approaches using deep learning. In particular, recurrent neural networks (RNNs) \cite{Hausknecht, Wierstra} and variational encoders \cite{Igl2018} have been used to model the uncertainty in PORL by capturing temporal dependencies. While these techniques have demonstrated impressive empirical results, they lack rigorous performance guarantees.

\section{Problem Description and Prior Results}

In this section, we introduce the problem formulation and review some well-known prior results; for basic POMDP theory, the reader is referred to \cite{Krishnamurthy_2016}. 

\subsection{POMDP Structure and Notations}
We consider a general finite-state POMDP model, which can be characterized as follows:
\begin{itemize}
    \item Consider a set of finite states $S$, with the state of the system at time $t$ denoted by $s_t \in \mathcal{S}$, which is \emph{not} observed. Also, assume that the initial state $s_0$ is sampled from a distribution $\mathcal{D}$, with the corresponding probability mass function represented by the vector $\boldsymbol{\pi}_0$.
    \item At any time $t$, the agent chooses an action $a_t = a$ from a finite set of possible actions $\mathcal{A}$, where, for simplicity, we assume that all actions from $\mathcal{A}$ are feasible for every state. The state then evolves according to a transition probability
    \begin{equation}\nonumber
         \mathcal{P}(s' \mid  s,a)=\mathbb{P}(s_{t+1} = s'  \mid   s_t = s, a_t = a).
    \end{equation}
    Moreover, the agent receives a reward $r(s,a)$, which is the reward obtained if performing an action $a$ in state $s$, where we assume finite reward, i.e., there exists $\bar{r}$ such that $|r(s,a)| \leq \bar{r} \ \forall s, a$.
    \item Since states cannot be observed, we assume that information about them at each time $t$ is obtained through a noisy discrete observation channel whose output $y_t \in \mathcal{Y}$ is chosen according to a conditional distribution $\Phi(y \mid  s) := \mathbb{P}(y_t = y  \mid   s_t = s),$ where $\mathcal{Y}$ is a finite set of observations. Therefore, after agent takes action $a_t$, it receives a reward $r_t$ and observes $y_{t+1}$.
    \item Let $H_t := \{a_0,y_1,a_1,\ldots,a_{t-1},y_t\}\in \mathbb{H}$ be the entire observed history up to time $t$, where $\mathbb{H}$ refers to the set of all possible histories (of any length) of the POMDP. Given a history $H\in \mathbb{H}$ and an initial distribution $\boldsymbol{\pi}_0$ over the states, we denote the probability of being in state $s$ by $\pi(s\mid H) = \mathbb{P}(s \mid  H,\boldsymbol{\pi}_0)$ and define the \emph{belief state} given history $H$ by $\boldsymbol{\pi}(H) = \begin{pmatrix} \pi(s \mid H) \end{pmatrix}_{s \in \mathcal{S}}$. In particular, the belief state\footnote{When there is no ambiguity, we drop the dependency of belief state on the history and write $\boldsymbol{\pi}=\boldsymbol{\pi}(H)$.} $\boldsymbol{\pi}$ belongs to $\mathbb{B} \subseteq \Sigma(\mathcal{S}),$ where $\Sigma(\mathcal{S}) := \{ x\in \mathbb{R}_+^{|\mathcal{S}|}: \sum_{i=1}^{|\mathcal{S}|} x_i = 1 \},$ and $\mathbb{B}$ refers to the set of all possible belief states that can be realized using a history $H\in\mathbb{H}$. Now, if at time $t,$ the agent takes an action $a_t=a$ and observes $y_t=y$, by Bayes' rule, the belief state can be updated as follows
    \begin{align}
    \label{Eq: belief update}
        \pi(s \mid  H_t)= \frac{\sum_{s'} \pi(s' \mid  H_{t-1}) \mathcal{P}(s \mid  s',a) \Phi(y \mid  s)}{\sum_{s''}\sum_{s'}  \pi(s' \mid  H_{t-1}) \mathcal{P}(s'' \mid  s',a) \Phi(y \mid  s'')}, 
    \end{align}
where $H_{t-1} = H_t \setminus \{y_t, a_t\}$. Therefore, any belief state $\boldsymbol{\pi}(H)$ can be calculated recursively from the history $H$ and using the belief update rule \eqref{Eq: belief update} with the belief at $t=0$ given by $\boldsymbol{\pi}_0.$ 
\item The agent's goal is to learn the optimal policy (defined more precisely in the next subsection) through sequential interactions with the environment.  
\end{itemize}

\subsection{Belief State MDP Representation of the POMDP}

We note that a POMDP can be reduced to a fully observed MDP by considering the belief states as the states of the MDP. This is discussed in detail below:
\begin{itemize}
 \item For any belief state $\boldsymbol{\pi}$, let $r(\boldsymbol{\pi},a) := \sum_s \pi(s) r(s,a)$. Thus, the POMDP reduces to a fully observed MDP where $(\boldsymbol{\pi},a)\in \mathbb{B}\times \mathcal{A}$ represents a state-action pair and $r(\boldsymbol{\pi},a)$ is the corresponding reward for that state-action pair with the state transition law given by Eq. \eqref{Eq: belief update}. 
    \item For any belief state $\boldsymbol{\pi} \in \mathbb{B}$ and any policy $\mu$, we define the value function $V^{\mu}(\boldsymbol{\pi})$ as
    \begin{equation}
        V^{\mu}(\boldsymbol{\pi}) := \mathbb{E}\big[\sum_{t=0}^{\infty} \gamma^t r(\boldsymbol{\pi}_t,a_t)  \mid   \boldsymbol{\pi}_0 = \boldsymbol{\pi}\big],
    \end{equation}
    where the belief states $\boldsymbol{\pi}_t:=\boldsymbol{\pi}(H_t)$ evolve using the update rule \eqref{Eq: belief update}, when the actions are taken according to $a_t \sim \mu(a \mid  \boldsymbol{\pi}_t) \ \forall t\ge 0$ with the initial belief state $\boldsymbol{\pi}_0 = \boldsymbol{\pi}$. Here, $\gamma\in [0,1)$ is a discount factor. Our goal is to find an optimal policy $\mu^*(a\mid\boldsymbol{\pi})$ corresponding to each belief state $\boldsymbol{\pi}$ to maximize the expected cumulative discounted rewards $V^{\mu}(\boldsymbol{\pi}).$
    \item For any belief state $\boldsymbol{\pi} \in \mathbb{B}$ and for a policy $\mu$, we define the Q-value function as
    \begin{equation}\nonumber
        Q^{\mu}(\boldsymbol{\pi},a) := \mathbb{E}\big[\sum_{t=0}^{\infty} \gamma^t r(\boldsymbol{\pi}_t,a_t)  \mid   \boldsymbol{\pi}_0 = \boldsymbol{\pi},a_0 = a\big],
    \end{equation}
        where the belief states $\boldsymbol{\pi}_t$ evolve using the update rule  \eqref{Eq: belief update} and the actions taken according to $a_t \sim \mu(a \mid  \boldsymbol{\pi}_t)\ \forall t\geq 1.$
    \end{itemize}

\subsection{Bellman's Optimality and Uniqueness of Solution}

Here, we first present a result from \cite{Krishnamurthy_2016} showing that the optimal value function satisfies \emph{Bellman's Optimality Equation}, which can also be used to establish the existence and uniqueness of an optimal solution to the POMDP.
\begin{theorem}
\label{thm: optimality}
    Consider an infinite horizon discounted reward POMDP with discount factor $\gamma \in [0,1)$, and finite state and action spaces $\mathcal{S}$ and $\mathcal{A}$, respectively. Then, 
    \begin{enumerate}
        \item For any belief state $\boldsymbol{\pi} \in \mathbb{B}$, the optimal expected cumulative discounted reward $\max_\mu V^{\mu}(\boldsymbol{\pi})$ is achieved by a stationary deterministic Markovian policy $\mu^*$.
        \item For any belief state $\boldsymbol{\pi}(H) \in \mathbb{B}$, the optimal policy $\mu^*$ and the optimal value function $V^{\mu*}$ satisfy the \emph{Bellman's Optimality Equation:} \footnote{The symbol $\,\|\,$ denotes cocatenation. For example if $H = \{y_0,a_0,y_1,a_1\},$ then, $H \,\|\, \{y,a\} = \{y_0,a_0,y_1,a_1,y,a\}$}
        \begin{align}\label{eq:belman-equation}
            V^{\mu^*}\big(\boldsymbol{\pi}(H)\big) = \max_{a\in \mathcal{A}}\Big[\sum_{s\in \mathcal{S}} \boldsymbol{\pi}(s|H) r(s,a) + \gamma \sum_{y\in \mathcal{Y}} V^{\mu^*}\big(\boldsymbol{\pi}(H \,\|\, \{y,a\})\big) \sigma(\boldsymbol{\pi}(H),y,a) \Big],
        \end{align}
        where $\sigma(\boldsymbol{\pi}(H),y,a)$ denotes the probability of the observation being $y$ conditioned on the previous belief state being $\boldsymbol{\pi}(H)$ and action $a$ is taken, which can be calculated explicitly as $\sigma(\boldsymbol{\pi}(H),y,a) = \sum_{s,s'} \Phi(y \mid  s') \mathcal{P}(s' \mid  s,a) \pi(s|H)$.
        \item There always exists a unique solution to the Bellman Optimality Equation \eqref{eq:belman-equation}.
    \end{enumerate}
\end{theorem}

Let us denote the Bellman optimality operator in \eqref{eq:belman-equation} by $T(\cdot)$. Also, let $\mu^*$ be the optimal policy and $V^* = V^{\mu^*}$ be the optimal value function. Then, the optimal value function $V^*$ satisfies the fixed-point equation 
\begin{align}\nonumber
    V^* = T(V^*).
\end{align}

Therefore, using Theorem \ref{thm: optimality}, for any belief state $\boldsymbol{\pi}=\boldsymbol{\pi}(H) \in \mathbb{B}$, we can write
\begin{align*}
    V^*\big( \boldsymbol{\pi}\big)= \max_{a\in \mathcal{A}} \Big[r( \boldsymbol{\pi},a)+ \gamma \sum_{y\in \mathcal{Y}} V^*\big( \boldsymbol{\pi}(H\,\|\,\{a,y\})\big) \sigma(\boldsymbol{\pi},y,a)  \Big].
\end{align*}

Note that reducing the POMDP to an MDP using the belief-state formulation does not solve the problem of finding the optimal policy using reinforcement learning. This is because the belief states depend on the history, whose length increases with each time step. In the following section, we present a scalable approach to bypass this intractability in solving infinite horizon POMDPs.

\section{Algorithm and Theoretical Results}

In this section, we describe our solution approach, develop an algorithm for approximately solving POMDPs, and provide theoretical guarantees on the performance of the algorithm.

\subsection{Defining an Approximate MDP Using Finite Observations}
\begin{definition}
    Given a POMDP and a fixed constant $l \in \mathbb{N}$, let $\mathbb{H}_{\leq l}$ be the set of all histories of length at most $l$. We refer to any element of $\mathbb{H}_{\leq l}$ as a \textit{Superstate}. Moreover, we define $\mathcal{G}: \mathbb{H} \rightarrow \mathbb{H}_{\leq l}$ as the grouping operator if for any finite-length history $H \in \mathbb{H}$, it returns the \textit{superstate} obtained by truncating $H$ to its last $l$ action-observation elements. In particular, if $H_t= \big\{a_0,y_1,a_1,\ldots,a_{t-1},y_t\}, \mbox{then}$\footnote{By notation $y_{r:t}=\{y_r,y_{r+1},\ldots,y_t\}$ and $a_{r:t}=\{a_r,a_{r+1},\ldots,a_t\}$ for any integers $r\leq t$.}
    \begin{align}\nonumber
        \mathcal{G}(H_t) = \big\{y_{\max\{1,t-l+1\}:t}, a_{\max\{0,t-l\}:t-1}\big\}, \quad \forall \ t \geq 1.
    \end{align}
\end{definition}

Next, we construct an MDP which consists of states corresponding to the Superstates. Since the number of Superstates is finite, our resulting Superstate MDP is a finite-state MDP, which can then be solved using existing RL techniques for finite-state MDPs in the literature. The Superstate MDP is defined as follows. For every pair of Superstates $B, B' \in \mathbb{H}_{\leq l}$, define
\begin{align}
    \label{Eq: superstate reward}
    \tilde{r}(B,a) &= \sum_{s} \pi(s \mid B) \cdot r(s,a),  \\
    \tilde{\mathcal{P}}(B'\mid B,a) &= \sum_{y,s,s'} \mathbb{I}[\mathcal{G}\big(B \,\|\, \{y,a\}\big) = B']\cdot \Phi(y\mid s')\mathcal{P}(s' \mid s,a) \pi(s \mid B), \label{Eq: superstate transition prob}
\end{align}
where $\mathbb{I}[\cdot]$ is the logic indicator function and $B \,\|\, \{y,a\}$ means that the action-observation pair $\{y,a\}$ is concatenated to the end of the superstate $B$. Since the action and state spaces are bounded, an optimal value function always exists for a discounted MDP with bounded rewards. Let \(\tilde{V}(\cdot)\) be the optimal value function for the Superstate MDP, which satisfies the following fixed-point equation
\begin{align}\nonumber
    \tilde{V} = \tilde{T}(\tilde{V}),
\end{align}
where $\tilde{T}$ is the Bellman's optimality operator for the corresponding Superstate MDP. Therefore, for any superstate $B\in \mathbb{H}_{\leq l}$, we have
\begin{align}\nonumber
    \tilde{T}(\tilde{V}\big(B)\big) 
    = \max_{a\in \mathcal{A}} \Big [ \tilde{r}(B,a) + \gamma \sum_{B'}  \tilde{\mathcal{P}}(B' \mid  B,a) \tilde{V}\big(B'\big) \Big].
\end{align}

Now, for a policy $\mu$ that is stationary with respect to the superstates, let us define the Bellman's operator $\tilde{T}^{\mu}$ by
\begin{align}\nonumber
    \tilde{T}^{\mu}(V)(B)
    = \mathbb{E}_{a \sim \mu(\cdot \mid B)}\Big [ \tilde{r}(B,a) + \gamma \sum_{B'} \tilde{\mathcal{P}}(B' \mid  B,a) V(B') \Big] \ \ \forall B \in \mathbb{H}_{\leq l}.   
\end{align}
Then, the value function of the Superstate MDP with respect to policy $\mu$, denoted by $\tilde{V}^{\mu}$, must satisfy the fixed-point Bellman's operator, i.e.,  
\begin{equation}\nonumber
    \tilde{V}^{\mu} = \tilde{T}^{\mu}(\tilde{V}^{\mu}).
\end{equation}

\subsection{How good is the approximate MDP?}

In this section, we compare optimal value function of the POMDP, denoted by $V^*$, with optimal value function of the Superstate MDP, denoted by $\tilde{V}$. We will show that error between these two optimal value functions decays exponentially with the length of the truncated history $l$. To that end, we first state a standard assumption in filtering theory, i.e., the \emph{Uniform Filter Stability Condition}, also used in prior work \cite{van2008hidden}, which would be needed for our analysis to relate $V^*$ and $\tilde{V}$. This condition
ensures sufficient mixing, preventing the system from being trapped in a subset of states.

\begin{assumption}\label{as:UFSC}[Uniform Filter Stability Condition]
Given any $\pi,\pi' \in \Sigma(\mathcal{S})$, and any $a,y$, let $K_{a,y}$ be an operator such that
\begin{equation*}
(K_{a,y} \otimes v) (s)=  \frac{\sum_{s'} v(s') \mathcal{P}(s \mid  s',a) \Phi(y \mid  s)}{\sum_{s''}\sum_{s'}  v(s') \mathcal{P}(s'' \mid  s',a) \Phi(y \mid  s'')}
\end{equation*}
Then, there exists $\rho\in (0,1)$ such that for any $\pi,\pi' \in \Sigma(\mathcal{S})$, we have\footnote{$\|\cdot\|_{TV}$ is the total variation norm}
    \begin{align*}
         \| K_{a,y} \otimes \pi - K_{a,y} \otimes \pi' \|  _{TV} \leq (1-\rho) \cdot  \|\pi-\pi'\|_{TV} \ \ \ \forall a\in \mathcal{A}, y\in \mathcal{Y}.
    \end{align*}
\end{assumption}

\subsection{Intuition about Assumption \ref{as:UFSC}}
Assumption \ref{as:UFSC}, i.e., the Uniform Filter Stability Condition means that every new observation sufficiently informs the agent to reduce differences between any two prior beliefs — ensuring the belief state “forgets” initial uncertainty and the filtering process is well-behaved.

Now to ensure that the filter stability condition holds we need

1) \textbf{Sufficiently Mixing State Transitions}: The transition kernel $\mathcal{P}(s'\mid s,a)$ should be mixing, meaning from any state s, there's a positive probability to reach many other states s' over time. 

2) \textbf{Non-deterministic and Informative Observations}: The observation kernel $\phi(y\mid a)$ must be non-deterministic but sufficiently informative, i.e., observations must have some noise or randomness.

To check whether the filter stability condition holds, one sufficient condition uses the Dobrushin Coefficients of the Transition kernel and the Observation Kernel. 

From Theorem 5 of \cite{kara_dob}, if $(1-\delta(\mathcal{P}))(1-\delta(\phi)) < 1$, then the filter stability condition holds. Here $\delta(\mathcal{P})$ is the minimum Dobrushin coeffient of the Transition kernel across all possible actions and $\delta(\phi)$ is the Dobrushin coefficient of the Observation kernel.

Let \( P \) be a row-stochastic matrix over a finite state space \( \mathbb{S} \). The Dobrushin coefficient \( \delta(P) \) is defined as:
\[
\delta(P(.\mid.,a)) = \inf_{x, y \in \mathbb{S}} \sum_{z \in \mathbb{S}} \min\left( P(z\mid x,a), P(z \mid y,a) \right)
\]
This quantifies the minimum overlap between any two rows of the matrix \( P \).

Therefore, this quantity measures how similar the rows of the matrix are; smaller values indicate more mixing and hence, stronger contraction properties. As a result, when there is sufficient mixing—i.e., when every state has a non-trivial probability of transitioning to several other states or when every observation can be generated from multiple underlying states—the system exhibits filter stability.

This scenario is common in highly noisy or dynamic environments, where the belief update is less sensitive to the prior and more influenced by new observations. We also present in the Appendix \ref{toy-example} a model for a practical example where Assumption \ref{as:UFSC} holds. 

Additionally, for applications that do not satisfy Assumption \ref{as:UFSC}, one could consider a multi-step variant where the system exhibits contraction after every $k$ steps. We believe our results could be extended under this weaker assumption, making it a promising direction for future work.

Now, using this Uniform Filter Stability Condition, we prove Lemma \ref{lemma: belief approx} \footnote{Proof of this lemma, as well as all subsequent lemmas and theorems, can be found in the supplementary document}, which shows that all belief states corresponding to the same Superstate are close to each other in terms of total variation distance. Further, the distance decays exponentially fast with the length of the truncated history. 
\begin{lemma}
\label{lemma: belief approx}
Let $ H$ and $ H'$ be two different histories corresponding to the same superstate, i.e., $\mathcal{G}( H) = \mathcal{G}( H')$. Then, under Assumption \ref{as:UFSC}, we have
\begin{equation}
\label{Eq: beliefs approx}
     \big\|   \boldsymbol{\pi}(H) -  \boldsymbol{\pi}(H') \big\|  _{TV} \leq (1-\rho)^l.
\end{equation}
\end{lemma}

If two histories belong to the same superstate, their sequences of $(\text{action}, \text{observation})$ pairs over the past $l$ steps are identical. Lemma \ref{lemma: belief approx} leverages the intuition that sufficiently informative observations allow the recent history to capture the current system state. Consequently, starting from different initial beliefs, the combination of informative observations and strong mixing ensures that, after enough time, the resulting belief states converge to similar distributions.

Next, we establish an algebraic inequality that will be used to obtain tighter optimality bounds.

\begin{lemma}
\label{lemma: norm bound}
    Let $\boldsymbol{a}, \boldsymbol{b}, \boldsymbol{c}, \boldsymbol{d} \in \mathbb{R}^{m}_+$ be positive vectors such that $\sum_{i=1}^m a_i = \sum_{i=1}^m c_i = 1$. Then,
    \begin{equation}
         \big|\sum_{i=1}^m a_i b_i - \sum_{i=1}^m c_i d_i \big|   \leq \frac{ \|  a-c \|  _1}{2} \max( \|  b \|  _{\infty}, \|  d \|  _{\infty}) 
        +  \|  b-d \|  _{\infty} - \frac{ \|  a-c \|  _1}{4}   \|  b-d \|  _{\infty}. 
    \end{equation}
\end{lemma}

Finally, using the above lemmas, we can state one of our main results.
\begin{theorem}
\label{thm: main result}
Let \( V^* \) be the optimal value function corresponding to the POMDP (or the belief state MDP), and let \( \tilde{V} \) be the optimal value function corresponding to the Superstate MDP. Then, under Assumption \ref{as:UFSC} for every history \( H \in \mathbb{H} \) with the corresponding belief state \( \boldsymbol{\pi}(H) \in \mathbb{B} \) and the superstate \( \mathcal{G}(H) \in \mathbb{H}_{\leq l} \), we have
\begin{equation}
\label{eq: value fn bound}
     \big\|   V^*\big(\boldsymbol{\pi}(H) \big) - \tilde{V}\big(\mathcal{G}(H) \big) \big\|  _{\infty} 
    \leq \frac{2\bar{r}(1-\rho)^l  }{1-\gamma}  + \frac{ 2 \bar{r} \gamma (1-\rho)^l }{(1-\gamma)\big((1-\gamma) + \gamma(1-\rho)^l \big)} 
    := \xi^{\textrm{SMDP}}_{\textrm{POMDP}}.
\end{equation} 
\end{theorem}
Thus, the above result relates the optimal value functions of the POMDP and the Superstate MDP. Note that the difference between the two value functions decreases exponentially with the length of the truncated history \( l \), with the error effectively becoming \( 0 \) as \( l \to \infty \). Next, we propose an algorithm to learn the optimal policy corresponding to the Superstate MDP.


\subsection{An Approximate Policy Optimization Algorithm to Learn the Optimal Policy corresponding to the Superstate MDP}
To learn the optimal value function for the Superstate MDP, one might consider standard reinforcement learning techniques such as Policy Optimization, which involves alternately learning the Q-function for a fixed policy and updating the policy. However, this learning process is not straightforward because the samples obtained at any time $t$ correspond to the actual belief state \( \boldsymbol{\pi}(H_t) \), rather than the Superstate \( \mathcal{G}(H_t) \). These issues due to sampling mismatch make the analysis of the TD-learning part of the algorithm non-trivial.

Additionally, we also consider the linear function approximation setting, where, given feature set $\Phi = \{\phi(B,a) \in \mathbb{R}^d: B \in \mathbb{H}_{\leq l}, \ a \in \mathcal{A} \},$ we aim to find the best $\theta \in \mathcal{B}(R)$\footnote{$\mathcal{B}(R)$ denotes a $d$-dimensional Euclidean ball of radius $R$.} for some $R > 0$, such that $Q(B,a) = \phi^T(B,a) \theta.$\footnote{Here, the superscript $T$ refers to the transpose of the vector $\phi$. } Further, since the feature vectors are bounded, we assume, without loss of generality, that $\|\phi(B,a)\|_2 \leq 1.$ Note that for the function approximation, we constrain the parameters \( \theta \) to lie within a ball of finite radius \( R \). This is done to simplify the analysis of the function approximation part of the algorithm. The case without such a projection can be analyzed as in \cite{srikant-ying,mitra2024simple}.
\begin{algorithm}\caption{An Approximate TD Learning Algorithm for Superstate MDP}\label{alg: TD algo}
\begin{algorithmic}
 \State\textbf{Input:} A fixed policy $\mu(\cdot \mid B)$, which is stationary with respect to the Superstates $B \in \mathbb{H}_{\leq l},$ 
\State Discount factor $\gamma$, projection radius $R > 0$, stepsize sequence $\{\epsilon_t\}$, total iterations $\tau+l'$
 \State Initialize $\theta_l$ randomly in $\mathcal{B}(R)$
     \State Sample $s_0 \sim \mathcal{D}$ and set $H_0 = \{\}$
\State    \textbf{For each} $t=0,1,\ldots,\tau+l'-1$ \textbf{do}
    \State \hspace{1em}     Select action $a_t$ according to the policy $\mu(\cdot \mid  \mathcal{G}(H_t))$
      \State \hspace{1em}     Receive reward $r_t$ and observe $y_{t+1}$
       \State \hspace{1em}    Update the history $H_{t+1} = H_t \,\|\, \{a_t, y_{t+1}\}$
         Select action $a_{t+1}$ according to the policy $\mu(\cdot \mid  \mathcal{G}(H_{t+1}))$
    \State \hspace{1em}       \textbf{If} $t \geq l'$ \textbf{then}
    \State \hspace{2em}       $\theta_{t+1/2} = \theta_t + \epsilon_t\big[r_t + \gamma \phi^T\big(\mathcal{G}(H_{t+1}),a_{t+1}\big)\theta_t - \phi^T\big(\mathcal{G}(H_t),a_t\big)\theta_t\big]\cdot \phi(\mathcal{G}\big(H_t),a_t\big)$
   \State \hspace{2em}        $\theta_{t+1} = \mathrm{Proj}_{\mathcal{B}(R)}(\theta_{t+1/2})$
      \State \hspace{1em}     \textbf{end if}
      \State   \textbf{end}
\State \textbf{Output:} $\bar{Q}^{\pi}(B,a) = \Phi^T(B,a) \theta_{\tau+l'}$.
\end{algorithmic}
\end{algorithm}

We now use a standard Temporal Difference (TD) learning algorithm to learn the Q-function of the Superstate MDP corresponding to a fixed policy $\mu.$ The main idea is to perform a TD-update at every time $t$, to the value function of $\mathcal{G}(H_t)$ using the reward $r_t$ and subsequent observation $y_{t+1}$. This process is summarized in Algorithm 1. While the TD learning algorithm is standard, we note that the model to which the algorithm is applied is not standard: we apply the algorithm pretending that the underlying model is an MDP while the true model is a POMDP. We leverage a key insight: if two belief states are close in total variation distance, their reward and transition functions can be proved to be close, allowing us to perform TD-learning by
pretending that the underlying model is Superstate MDP while the true model is actually a POMDP.


Next, we show that under Assumption 1 and with sufficient exploration by the policies, the Superstate MDP admits a contraction mapping.
\begin{lemma}
\label{lemma: mixing}
    Let $\tilde{\mathcal{P}}^{\mu}\in \mathbb{R}_+^{|\mathbb{H}_{\leq l}|\times |\mathbb{H}_{\leq l}|\times |\mathcal{A}|}$ be the probability transition matrix defined by Eq. \eqref{Eq: superstate transition prob} and following policy $\mu$. Then, there exists a constant $\rho' \in (0, 1)$ such that for all pairs of distributions $d_1, d_2$ over the superstates such that $(d_1-d_2)(i) \neq 0 \ \forall \ i$ and for all policies with sufficient exploration $\mu(a \mid B) \geq \delta \ \forall \ a,B$, such that $(1-\rho)^l < \delta |\mathbb{A}|$ we have
    \begin{align*}
        \big\| \tilde{\mathcal{P}}^{\mu}d_1- \tilde{\mathcal{P}}^{\mu}d_2 \big\|_{TV} \leq (1-\rho') \| d_1-d_2 \|_{TV}.
    \end{align*}
\end{lemma}
Finally, we state the approximation bounds for the approximate TD-learning algorithm
\begin{lemma}
\label{Lemma: Approx TD}
Suppose Assumptions \ref{as:UFSC} holds and consider that $\mu(a \mid B) \geq \delta  \ \forall \ a,B$, such that \footnote{Given how the policy is chosen in Algorithm 2, a non-zero $\delta$ always exists} $(1-\rho)^l < \delta |\mathbb{A}|$ for all policies $\mu$. Let $\bar{Q}_{\tau+l'}^{\mu}$ denote the Q-function obtained by running Algorithm 1 for $\tau > 4(1-\gamma)^2$ iterations with a fixed stepsize $\epsilon_t = 1/\sqrt{\tau}$, and $l' = \frac{\log \tau}{2 \log(1-\rho')}$. Moreover, let $\tilde{Q}^{\mu}$ be the actual Q-function of the Superstate MDP corresponding to the policy $\mu$, which satisfies $\tilde{Q}^{\mu}(B,a) = \tilde{T}^{\mu} (\tilde{Q}^{\mu}(B,a))\ \forall B \in \mathbb{H}_{\leq l}, a \in \mathcal{A}$ and let $\hat{\theta} := \min_{\|\theta\| \leq R} \|\tilde{Q}^{\mu} - \Phi^T \theta\|^2$. Then,
\begin{align}\nonumber
    &\mathbb{E}\| \bar{Q}_{\tau+l'}^{\mu} - \tilde{Q}^{\mu} \|_{\infty} 
    \leq \|\Phi^T \hat{\theta} - \tilde{Q}^{\mu} \|_{\infty}+ \Big(1-\frac{2(1-\gamma)}{\sqrt{\tau}}\Big)^{\tau} \|\theta_l-\hat{\theta} \|_2 + \Big[\frac{1-(1-2(1-\gamma)/\sqrt{\tau})^{\tau}}{(1-\gamma)}\Big] \cr 
    & \times\Big[\frac{(\bar{r} + 2R)^2}{2\sqrt{\tau}}
     + \frac{C_2 (\bar{r} + 2R) \log \tau}{\sqrt{\tau}  \log(1-\rho')}+ (1-\rho')^{ \frac{\log \tau}{2 \log(1-\rho')}}\big(R\bar{r} + R^2(1+(1-\rho)\gamma)\big) \nonumber \\
     &  +  2R\bar{r}(1-\rho)^l + \frac{2}{\rho'}\big(1-\rho\big)^l\big(R\bar{r} + R^2(1+(1-\rho)\gamma)\big) \Big]\!:= \xi_{\textrm{TD-Error}}, \nonumber
\end{align}
\end{lemma}

\begin{algorithm}\caption{A Policy Optimization Based Algorithm to learn the Superstate MDP}\label{alg: main algo}
\begin{algorithmic}
\State Set $Q_0(B,a) = 0, \ \ \forall B\in \mathbb{H}_{\leq l}, a\in \mathcal{A}$
\State \textbf{For each} \ $i = 1,2,\ldots,M$ \textbf{do}
  \State \hspace{1em}   $\mu_i(a \mid B) \propto \ \exp\big(\eta \sum_{j=1}^{i} \bar{Q}^{\mu_{j-1}}_{\tau+l'}(B,a)\big)$
    \State \hspace{1em}    Initialize $\theta_l$ randomly in $\mathcal{B}(R)$
     \State \hspace{1em}   Sample $s_0 \sim \mathcal{D}$ and set $H_0^i = \{\}$
   \State \hspace{1em}    \textbf{For each} \ $t=0$ to $\tau+l'-1$ \textbf{do}
       \State \hspace{2em}     Select action $a_t$ according to policy $\mu_i(\cdot \mid  \mathcal{G}(H_t^i))$
     \State \hspace{2em}       Observe reward $r_t$ and the next observation $y_{t+1}$
       \State \hspace{2em}     Update the history $H^i_{t+1} = H_t^i \,\|\, \{a_t, y_{t+1}\}$
         \State \hspace{2em}   Select action $a_{t+1}$ according to the policy $\mu_i(\cdot \mid  \mathcal{G}(H^i_{t+1}))$
       \State \hspace{2em}    \textbf{If} \ $t \geq l'$ \textbf{then}
          \State \hspace{3em} $\theta_{t+1/2} = \theta_t + \epsilon_t\Big(r_t + \gamma \phi^T\big(\mathcal{G}(H^i_{t+1}),a_{t+1}\big)\theta_t - \phi^T\big(\mathcal{G}(H^i_t),a_t\big)\theta_t\Big)\phi(\mathcal{G}\big(H^i_t),a_t\big)$
     \State \hspace{3em}      $\theta_{t+1} = \mathrm{Proj}_{\mathcal{B}(R)}(\theta_{t+1/2})$
    \State \hspace{2em}       \textbf{end If}
     \State \hspace{1em}  \textbf{end}
     $\bar{Q}^{\mu_i}_{\tau+l'}(B,a) = \Phi^T(B,a) \theta_{\tau+l'}$
\State \textbf{end}
\end{algorithmic}
    \end{algorithm}

Next, we combine Algorithm 1 with the POLITEX algorithm from \cite{politex} for the policy update rule to learn the optimal policy. Note that the POLITEX algorithm is proposed for the average reward setting, whereas in this work, we extend their analysis to the discounted reward problem. The overall algorithm is outlined in Algorithm 2, where the inner loop performs TD learning, as described in Algorithm 1, while the outer loop performs policy updates using an exponential update rule which incorporates aggregate information from learned Q-functions.

Next, to evaluate the performance of our algorithm, we use the following notion of \emph{regret}. Similar definitions have also been used in \cite{he_neurips}. Regret is therefore defined as
\begin{equation}
    \label{Eq: regret defn}
    \mathcal{R}_T= \mathbb{E}\Big[\sum_{i=1}^{M} \sum_{j=0}^{\tau+l'-1}\Big(V^{\mu^*}(\boldsymbol{\pi}(H_0)) - \tilde{V}^{\mu_i}(\mathcal{G}(H_0))\Big)\Big] = (\tau +l') \sum_{i=1}^{M} \mathbb{E}\Big[V^{\mu^*}(\boldsymbol{\pi}(H_0)) - \tilde{V}^{\mu_i}(\mathcal{G}(H_0))\Big],
\end{equation}
where $V^{\mu^*}$ and $\tilde{V}^{\mu_i}$ are value functions of the actual POMDP and the Superstate MDP under optimal policy $\mu^*$ and policy $\mu_i$, respectively. Here, the number of policy updates \( M \) and the number of inner TD learning iterations in each episode \( \tau \) are chosen such that \( M(\tau + l') = T \). The intuition behind the definition is that, suppose the algorithm stops at the \( j \)-th inner iteration of the \( i \)-th policy update episode. Then, the error between the expected discounted reward corresponding to the optimal policy and the policy output by the algorithm is \( V^{\mu^*}(\boldsymbol{\pi}(H_0)) - \tilde{V}^{\mu_i}(\mathcal{G}(H_0)) \). Therefore, \( \mathcal{R}_T/T \) can also be viewed as the expected error incurred by the algorithm if it stops at a uniformly chosen random time.

The following theorem provides an analytical upper bound on the regret of our proposed algorithm.\footnote{A direct comparison of our bound with \cite{semih} is also provided in the Supplementary.}
\begin{theorem}
\label{thm: fin thm}
    Let $V^*$ be the optimal value function of the POMDP, and $\{\mu_i\}_{i=1}^M$ be the policies learned in Algorithm 2 at the corresponding discrete time intervals $t_i = [(i-1)(\tau+l')+1,i(\tau+l')],$ $i=1,\ldots,M$. Moreover, let the regret $\mathcal{R}_T$ be as defined in Eq. \eqref{Eq: regret defn}. Further, let $\tau = \sqrt{T}$ and thus $l' = \frac{\log T}{4 \log (1-\rho')}$ and $M = \frac{T}{(\tau+l')}$. Then, the regret is bounded as
    \begin{equation}
        \mathcal{R}_T \leq T \cdot (\xi_{\textrm{FA}} + \xi_{\textrm{HA}})  + \mathcal{O}(T^{3/4}\log T)
        \end{equation}
where
\begin{align}\nonumber
&\xi_{\textrm{FA}} = 2\sum_{i=1}^{M}\|\Phi^T \hat{\theta}_i - \tilde{Q}^{\mu_i} \|_{\infty} /M,\cr
&\xi_{\textrm{HA}} = \big(1-\rho\big)^l\Big[\frac{1-(1-2(1-\gamma)/\sqrt{\tau})^{\tau}}{(1-\gamma)}\Big] \cdot \Big(4 R\bar{r} +4/\rho'\big(R\bar{r} + R^2(1+(1-\rho)\gamma)\big)\big) \cr 
&\qquad+  \frac{2  \bar{r} }{(1-\gamma)}  + \frac{ 2  \bar{r} \gamma}{(1-\gamma)\big(2(1-\gamma) + (1-\rho)^l \gamma\big)}\Big),
\end{align}

\end{theorem}

Since Algorithm 2 is devised to optimize the Superstate MDP, a small regret implies that the realized trajectory of the algorithm under the actual POMDP is also close to its optimal value. 

Note that \( \xi_{\textrm{FA}} \) is the error due to linear function approximation which can be reduced by using a good set of feature vectors. Similarly, \( \xi_{\textrm{HA}} \) is the error due to approximating the history using a truncated history of length \( l \) (Superstate), which quantifies the tradeoff between increased complexity in terms of the number of states in the Superstate MDP and the approximation error.

\section{Conclusion}
We show that standard policy optimization algorithms can effectively approximate an optimal POMDP policy by modeling it as an MDP over finite histories, and provide convergence guarantees without the heavy computational cost or restrictive assumptions of prior methods. Our results also extend to the linear function approximation setting, ensuring scalability to large state spaces. Finally, we extend the POLITEX algorithm to the discounted reward setting and analyzed the regret with respect to the optimal POMDP value function. Overall, our work provides tighter theoretical guarantees, improved efficiency, and a more scalable solution for solving PORL problems. Future work could focus on tightening the approximation bounds by leveraging more expressive function approximators, such as LSTMs or Transformer-based architectures.

\begin{ack}
The work done in this paper was supported by NSF Grants CCF 22-07547, CCF 1934986, CNS 21-06801, CAREER Award EPCN-1944403, AFOSR Grant FA9550-23-1-0107 and ONR Grant N00014-19-12566.
\end{ack}

\bibliography{refs}

\begin{thebibliography}{52}
\providecommand{\natexlab}[1]{#1}
\providecommand{\url}[1]{\texttt{#1}}
\expandafter\ifx\csname urlstyle\endcsname\relax
  \providecommand{\doi}[1]{doi: #1}\else
  \providecommand{\doi}{doi: \begingroup \urlstyle{rm}\Url}\fi

\bibitem[Cassandra et~al.(1996)Cassandra, Kaelbling, and Kurien]{cass2}
A.~Cassandra, L.~Kaelbling, and J.~Kurien.
\newblock Acting uncertainty: Discrete bayesian models for mobile-robot navigation.
\newblock Technical report, USA, 1996.

\bibitem[Levinson et~al.(2011)Levinson, Askeland, Becker, Dolson, Held, Kammel, Kolter, Langer, Pink, Pratt, Sokolsky, Stanek, Stavens, Teichman, Werling, and Thrun]{Held-2011}
Jesse Levinson, Jake Askeland, Jan Becker, Jennifer Dolson, David Held, Soeren Kammel, J.~Zico Kolter, Dirk Langer, Oliver Pink, Vaughan Pratt, Michael Sokolsky, Ganymed Stanek, David Stavens, Alex Teichman, Moritz Werling, and Sebastian Thrun.
\newblock Towards fully autonomous driving: Systems and algorithms.
\newblock In \emph{Proceedings of IEEE Intelligent Vehicles Symposium (IV '11)}, pages 163 -- 168, June 2011.

\bibitem[Li et~al.(2010)Li, Chu, Langford, and Schapire]{News}
Lihong Li, Wei Chu, John Langford, and Robert~E. Schapire.
\newblock A contextual-bandit approach to personalized news article recommendation.
\newblock In \emph{Proceedings of the 19th International Conference on World Wide Web}, WWW '10, page 661–670, New York, NY, USA, 2010. Association for Computing Machinery.
\newblock ISBN 9781605587998.
\newblock \doi{10.1145/1772690.1772758}.
\newblock URL \url{https://doi.org/10.1145/1772690.1772758}.

\bibitem[Hauskrecht and Fraser(2000)]{HAUSKRECH}
Milos Hauskrecht and Hamish Fraser.
\newblock Planning treatment of ischemic heart disease with partially observable markov decision processes.
\newblock \emph{Artificial Intelligence in Medicine}, 18\penalty0 (3):\penalty0 221--244, 2000.
\newblock ISSN 0933-3657.
\newblock \doi{https://doi.org/10.1016/S0933-3657(99)00042-1}.
\newblock URL \url{https://www.sciencedirect.com/science/article/pii/S0933365799000421}.

\bibitem[Brown and Sandholm(2019)]{Brown2019}
Noam Brown and Tuomas Sandholm.
\newblock Superhuman ai for multiplayer poker.
\newblock \emph{Science}, 365:\penalty0 885 -- 890, 2019.
\newblock URL \url{https://api.semanticscholar.org/CorpusID:195892791}.

\bibitem[Smallwood and Sondik(1973)]{smallwood}
Richard~D. Smallwood and Edward~J. Sondik.
\newblock The optimal control of partially observable markov processes over a finite horizon.
\newblock \emph{Operations Research}, 21\penalty0 (5):\penalty0 1071--1088, 1973.
\newblock \doi{10.1287/opre.21.5.1071}.
\newblock URL \url{https://doi.org/10.1287/opre.21.5.1071}.

\bibitem[Åström(1965)]{ASTROM}
K.J Åström.
\newblock Optimal control of markov processes with incomplete state information.
\newblock \emph{Journal of Mathematical Analysis and Applications}, 10\penalty0 (1):\penalty0 174--205, 1965.
\newblock ISSN 0022-247X.
\newblock \doi{https://doi.org/10.1016/0022-247X(65)90154-X}.
\newblock URL \url{https://www.sciencedirect.com/science/article/pii/0022247X6590154X}.

\bibitem[Papadimitriou and Tsitsiklis(1999)]{papadimitriou}
Christos~H. Papadimitriou and John~N. Tsitsiklis.
\newblock The complexity of optimal queuing network control.
\newblock \emph{Mathematics of Operations Research}, 24\penalty0 (2):\penalty0 293--305, 1999.
\newblock \doi{10.1287/moor.24.2.293}.
\newblock URL \url{https://doi.org/10.1287/moor.24.2.293}.

\bibitem[Vlassis et~al.(2012)Vlassis, Littman, and Barber]{vlassis}
Nikos Vlassis, Michael~L. Littman, and David Barber.
\newblock On the computational complexity of stochastic controller optimization in pomdps.
\newblock \emph{ACM Trans. Comput. Theory}, 4\penalty0 (4), November 2012.
\newblock ISSN 1942-3454.
\newblock \doi{10.1145/2382559.2382563}.
\newblock URL \url{https://doi.org/10.1145/2382559.2382563}.

\bibitem[Jaakkola et~al.(1994)Jaakkola, Singh, and Jordan]{Jaakkola}
Tommi Jaakkola, Satinder~P. Singh, and Michael~I. Jordan.
\newblock Reinforcement learning algorithm for partially observable markov decision problems.
\newblock In \emph{Proceedings of the 7th International Conference on Neural Information Processing Systems}, NIPS'94, page 345–352, Cambridge, MA, USA, 1994. MIT Press.

\bibitem[Williams and Singh(1998)]{Williams}
John Williams and Satinder Singh.
\newblock Experimental results on learning stochastic memoryless policies for partially observable markov decision processes.
\newblock In M.~Kearns, S.~Solla, and D.~Cohn, editors, \emph{Advances in Neural Information Processing Systems}, volume~11. MIT Press, 1998.
\newblock URL \url{https://proceedings.neurips.cc/paper_files/paper/1998/file/1cd3882394520876dc88d1472aa2a93f-Paper.pdf}.

\bibitem[Azizzadenesheli et~al.(2016)Azizzadenesheli, Lazaric, and Anandkumar]{aziz}
Kamyar Azizzadenesheli, Alessandro Lazaric, and Animashree Anandkumar.
\newblock Reinforcement learning of pomdps using spectral methods.
\newblock In Vitaly Feldman, Alexander Rakhlin, and Ohad Shamir, editors, \emph{29th Annual Conference on Learning Theory}, volume~49 of \emph{Proceedings of Machine Learning Research}, pages 193--256, Columbia University, New York, New York, USA, 23--26 Jun 2016. PMLR.
\newblock URL \url{https://proceedings.mlr.press/v49/azizzadenesheli16a.html}.

\bibitem[Loch and Singh(1998)]{loch}
John Loch and Satinder~P. Singh.
\newblock Using eligibility traces to find the best memoryless policy in partially observable markov decision processes.
\newblock In \emph{Proceedings of the Fifteenth International Conference on Machine Learning}, ICML '98, page 323–331, San Francisco, CA, USA, 1998. Morgan Kaufmann Publishers Inc.
\newblock ISBN 1558605568.

\bibitem[Littman(1994)]{Littman}
Michael~L. Littman.
\newblock Memoryless policies: theoretical limitations and practical results.
\newblock In \emph{Proceedings of the Third International Conference on Simulation of Adaptive Behavior : From Animals to Animats 3: From Animals to Animats 3}, SAB94, page 238–245, Cambridge, MA, USA, 1994. MIT Press.
\newblock ISBN 0262531224.

\bibitem[Bertsekas and Tsitsiklis(1996)]{bertsekas}
Dimitri~P. Bertsekas and John~N. Tsitsiklis.
\newblock \emph{Neuro-Dynamic Programming}, volume~3 of \emph{Anthropological Field Studies}.
\newblock Athena Scientific, 1996.

\bibitem[Sutton and Barto(2018)]{sutton}
Richard~S. Sutton and Andrew~G. Barto.
\newblock \emph{Reinforcement Learning: An Introduction}.
\newblock A Bradford Book, Cambridge, MA, USA, 2018.
\newblock ISBN 0262039249.

\bibitem[Kara and Y\"{u}ksel(2023)]{kara}
Ali~Devran Kara and Serdar Y\"{u}ksel.
\newblock Convergence of finite memory q learning for pomdps and near optimality of learned policies under filter stability.
\newblock \emph{Mathematics of Operations Research}, 48\penalty0 (4):\penalty0 2066--2093, 2023.
\newblock \doi{10.1287/moor.2022.1331}.
\newblock URL \url{https://doi.org/10.1287/moor.2022.1331}.

\bibitem[Cayci et~al.(2024)Cayci, He, and Srikant]{semih}
Semih Cayci, Niao He, and R.~Srikant.
\newblock Finite-time analysis of natural actor-critic for pomdps.
\newblock \emph{SIAM Journal on Mathematics of Data Science}, 6\penalty0 (4):\penalty0 869--896, 2024.
\newblock \doi{10.1137/23M1587683}.
\newblock URL \url{https://doi.org/10.1137/23M1587683}.

\bibitem[Subramanian and Mahajan(2019)]{mahajan}
Jayakumar Subramanian and Aditya Mahajan.
\newblock Approximate information state for partially observed systems.
\newblock In \emph{2019 IEEE 58th Conference on Decision and Control (CDC)}, pages 1629--1636, 2019.
\newblock \doi{10.1109/CDC40024.2019.9029898}.

\bibitem[Abel et~al.(2016)Abel, Hershkowitz, and Littman]{abel}
David Abel, David Hershkowitz, and Michael Littman.
\newblock Near optimal behavior via approximate state abstraction.
\newblock In Maria~Florina Balcan and Kilian~Q. Weinberger, editors, \emph{Proceedings of The 33rd International Conference on Machine Learning}, volume~48 of \emph{Proceedings of Machine Learning Research}, pages 2915--2923, New York, New York, USA, 20--22 Jun 2016. PMLR.
\newblock URL \url{https://proceedings.mlr.press/v48/abel16.html}.

\bibitem[Abbasi-Yadkori et~al.(2019)Abbasi-Yadkori, Bartlett, Bhatia, Lazic, Szepesvari, and Weisz]{politex}
Yasin Abbasi-Yadkori, Peter Bartlett, Kush Bhatia, Nevena Lazic, Csaba Szepesvari, and Gellert Weisz.
\newblock {POLITEX}: Regret bounds for policy iteration using expert prediction.
\newblock In Kamalika Chaudhuri and Ruslan Salakhutdinov, editors, \emph{Proceedings of the 36th International Conference on Machine Learning}, volume~97 of \emph{Proceedings of Machine Learning Research}, pages 3692--3702. PMLR, 09--15 Jun 2019.
\newblock URL \url{https://proceedings.mlr.press/v97/lazic19a.html}.

\bibitem[Cassandra et~al.(1994)Cassandra, Kaelbling, and Littman]{Cassandra}
Anthony~R. Cassandra, Leslie~Pack Kaelbling, and Michael~L. Littman.
\newblock Acting optimally in partially observable stochastic domains.
\newblock In \emph{AAAI Conference on Artificial Intelligence}, 1994.
\newblock URL \url{https://api.semanticscholar.org/CorpusID:16792751}.

\bibitem[Lovejoy(1993)]{lovejoy}
William~S. Lovejoy.
\newblock Suboptimal policies, with bounds, for parameter adaptive decision processes.
\newblock \emph{Operations Research}, 41\penalty0 (3):\penalty0 583--599, 1993.
\newblock ISSN 0030364X, 15265463.
\newblock URL \url{http://www.jstor.org/stable/171857}.

\bibitem[Zhang and Liu(1996)]{Zhang}
Nevin~Lianwen Zhang and Wenju Liu.
\newblock Planning in stochastic domains: Problem characteristics and approximation.
\newblock 1996.
\newblock URL \url{https://api.semanticscholar.org/CorpusID:12681087}.

\bibitem[Krishnamurthy(2016)]{Krishnamurthy_2016}
Vikram Krishnamurthy.
\newblock \emph{Partially Observed Markov Decision Processes: From Filtering to Controlled Sensing}.
\newblock Cambridge University Press, 2016.

\bibitem[Murphy(2007)]{Murphy}
Kevin~P. Murphy.
\newblock A survey of pomdp solution techniques.
\newblock 2007.
\newblock URL \url{https://api.semanticscholar.org/CorpusID:261299515}.

\bibitem[Sung et~al.(2017)Sung, Salisbury, and Saxena]{sung}
Jaeyong Sung, J.~Kenneth Salisbury, and Ashutosh Saxena.
\newblock Learning to represent haptic feedback for partially-observable tasks.
\newblock In \emph{2017 IEEE International Conference on Robotics and Automation (ICRA)}, page 2802–2809. IEEE Press, 2017.
\newblock \doi{10.1109/ICRA.2017.7989326}.
\newblock URL \url{https://doi.org/10.1109/ICRA.2017.7989326}.

\bibitem[Yu(2005)]{yu}
Huizhen Yu.
\newblock A function approximation approach to estimation of policy gradient for pomdp with structured policies.
\newblock In \emph{Proceedings of the Twenty-First Conference on Uncertainty in Artificial Intelligence}, UAI'05, page 642–649, Arlington, Virginia, USA, 2005. AUAI Press.
\newblock ISBN 0974903914.

\bibitem[Cayci and Eryilmaz(2024)]{cayci2024RNN}
Semih Cayci and Atilla Eryilmaz.
\newblock Recurrent natural policy gradient for pomdps, 2024.
\newblock URL \url{https://arxiv.org/abs/2405.18221}.

\bibitem[Amiri and Magnússon(2024)]{amiri}
Mohsen Amiri and Sindri Magnússon.
\newblock On the convergence of td-learning on markov reward processes with hidden states.
\newblock In \emph{European Control Conference, ECC 2024, Stockholm, Sweden, June 25-28, 2024}, pages 2097--2104. IEEE, 2024.
\newblock ISBN 978-3-9071-4410-7.
\newblock \doi{10.23919/ECC64448.2024.10591108}.
\newblock URL \url{https://doi.org/10.23919/ECC64448.2024.10591108}.

\bibitem[McCallum(1993)]{McCallum1993}
Andrew McCallum.
\newblock Overcoming incomplete perception with utile distinction memory.
\newblock In \emph{International Conference on Machine Learning}, 1993.
\newblock URL \url{https://api.semanticscholar.org/CorpusID:17063561}.

\bibitem[Meuleau et~al.(1999)Meuleau, Peshkin, Kim, and Kaelbling]{meuleau}
Nicolas Meuleau, Leonid Peshkin, Kee-Eung Kim, and Leslie~Pack Kaelbling.
\newblock Learning finite-state controllers for partially observable environments.
\newblock In \emph{Proceedings of the Fifteenth Conference on Uncertainty in Artificial Intelligence}, UAI'99, page 427–436, San Francisco, CA, USA, 1999. Morgan Kaufmann Publishers Inc.
\newblock ISBN 1558606149.

\bibitem[Du et~al.(2019)Du, Krishnamurthy, Jiang, Agarwal, Dudik, and Langford]{simon}
Simon Du, Akshay Krishnamurthy, Nan Jiang, Alekh Agarwal, Miroslav Dudik, and John Langford.
\newblock Provably efficient {RL} with rich observations via latent state decoding.
\newblock In Kamalika Chaudhuri and Ruslan Salakhutdinov, editors, \emph{Proceedings of the 36th International Conference on Machine Learning}, volume~97 of \emph{Proceedings of Machine Learning Research}, pages 1665--1674. PMLR, 09--15 Jun 2019.
\newblock URL \url{https://proceedings.mlr.press/v97/du19b.html}.

\bibitem[Efroni et~al.(2022)Efroni, Jin, Krishnamurthy, and Miryoosefi]{Efroni2022}
Yonathan Efroni, Chi Jin, Akshay Krishnamurthy, and Sobhan Miryoosefi.
\newblock Provable reinforcement learning with a short-term memory.
\newblock \emph{ArXiv}, abs/2202.03983, 2022.
\newblock URL \url{https://api.semanticscholar.org/CorpusID:246652495}.

\bibitem[Wang et~al.(2022)Wang, Cai, Yang, and Wang]{lingwang}
Lingxiao Wang, Qi~Cai, Zhuoran Yang, and Zhaoran Wang.
\newblock Embed to control partially observed systems: Representation learning with provable sample efficiency, 05 2022.

\bibitem[Liu et~al.(2022)Liu, Chung, Szepesvári, and Jin]{liu2022}
Qinghua Liu, Alan Chung, Csaba Szepesvári, and Chi Jin.
\newblock When is partially observable reinforcement learning not scary?, 2022.
\newblock URL \url{https://arxiv.org/abs/2204.08967}.

\bibitem[Jin et~al.(2020)Jin, Kakade, Krishnamurthy, and Liu]{Jin2020}
Chi Jin, Sham~M. Kakade, Akshay Krishnamurthy, and Qinghua Liu.
\newblock Sample-efficient reinforcement learning of undercomplete pomdps.
\newblock \emph{ArXiv}, abs/2006.12484, 2020.
\newblock URL \url{https://api.semanticscholar.org/CorpusID:219965941}.

\bibitem[Hausknecht and Stone(2015)]{Hausknecht}
Matthew~J. Hausknecht and Peter Stone.
\newblock Deep recurrent q-learning for partially observable mdps.
\newblock \emph{ArXiv}, abs/1507.06527, 2015.
\newblock URL \url{https://api.semanticscholar.org/CorpusID:8696662}.

\bibitem[Wierstra et~al.(2007)Wierstra, F{\"o}rster, Peters, and Schmidhuber]{Wierstra}
Daan Wierstra, Alexander F{\"o}rster, Jan Peters, and J{\"u}rgen Schmidhuber.
\newblock Solving deep memory pomdps with recurrent policy gradients.
\newblock In \emph{International Conference on Artificial Neural Networks}, 2007.
\newblock URL \url{https://api.semanticscholar.org/CorpusID:14039355}.

\bibitem[Igl et~al.(2018)Igl, Zintgraf, Le, Wood, and Whiteson]{Igl2018}
Maximilian Igl, Luisa~M. Zintgraf, Tuan~Anh Le, Frank Wood, and Shimon Whiteson.
\newblock Deep variational reinforcement learning for pomdps.
\newblock In \emph{International Conference on Machine Learning}, 2018.
\newblock URL \url{https://api.semanticscholar.org/CorpusID:46955236}.

\bibitem[van Handel(2008)]{van2008hidden}
Ramon van Handel.
\newblock Hidden markov models.
\newblock \emph{Unpublished lecture notes}, 2008.

\bibitem[Kara and Yuksel(2020)]{kara_dob}
Ali Kara and Serdar Yuksel.
\newblock Near optimality of finite memory feedback policies in partially observed markov decision processes, 10 2020.

\bibitem[Srikant and Ying(2019)]{srikant-ying}
R.~Srikant and Lei Ying.
\newblock Finite-time error bounds for linear stochastic approximation andtd learning.
\newblock In Alina Beygelzimer and Daniel Hsu, editors, \emph{Proceedings of the Thirty-Second Conference on Learning Theory}, volume~99 of \emph{Proceedings of Machine Learning Research}, pages 2803--2830. PMLR, 25--28 Jun 2019.
\newblock URL \url{https://proceedings.mlr.press/v99/srikant19a.html}.

\bibitem[Mitra(2024)]{mitra2024simple}
Aritra Mitra.
\newblock A simple finite-time analysis of td learning with linear function approximation.
\newblock \emph{arXiv preprint arXiv:2403.02476}, 2024.

\bibitem[He et~al.(2021)He, Zhou, and Gu]{he_neurips}
Jiafan He, Dongruo Zhou, and Quanquan Gu.
\newblock Nearly minimax optimal reinforcement learning for discounted mdps.
\newblock In M.~Ranzato, A.~Beygelzimer, Y.~Dauphin, P.S. Liang, and J.~Wortman Vaughan, editors, \emph{Advances in Neural Information Processing Systems}, volume~34, pages 22288--22300. Curran Associates, Inc., 2021.
\newblock URL \url{https://proceedings.neurips.cc/paper_files/paper/2021/file/bb57db42f77807a9c5823bd8c2d9aaef-Paper.pdf}.

\bibitem[Cesa-Bianchi and Lugosi(2006)]{Bianchi}
Nicolò Cesa-Bianchi and Gábor Lugosi.
\newblock \emph{Prediction, Learning, and Games}.
\newblock 01 2006.
\newblock ISBN 978-0-521-84108-5.
\newblock \doi{10.1017/CBO9780511546921}.

\bibitem[Kakade and Langford(2002)]{kakade2002approximately}
Sham Kakade and John Langford.
\newblock Approximately optimal approximate reinforcement learning.
\newblock In \emph{Proceedings of the Nineteenth International Conference on Machine Learning}, pages 267--274, 2002.

\bibitem[Bakker(2001)]{bakker}
Bram Bakker.
\newblock Reinforcement learning with long short-term memory.
\newblock In T.~Dietterich, S.~Becker, and Z.~Ghahramani, editors, \emph{Advances in Neural Information Processing Systems}, volume~14. MIT Press, 2001.
\newblock URL \url{https://proceedings.neurips.cc/paper_files/paper/2001/file/a38b16173474ba8b1a95bcbc30d3b8a5-Paper.pdf}.

\bibitem[Paischer et~al.(2022)Paischer, Adler, Patil, Bitto-Nemling, Holzleitner, Lehner, Eghbalzadeh, and Hochreiter]{expt1}
Fabian Paischer, Thomas Adler, Vihang Patil, Angela Bitto-Nemling, Markus Holzleitner, Sebastian Lehner, Hamid Eghbalzadeh, and Sepp Hochreiter.
\newblock History compression via language models in reinforcement learning, 05 2022.

\bibitem[Li et~al.(2024)Li, Zhao, Wei, Xing, and Xiang]{expt2}
Jinqiu Li, Enmin Zhao, Tong Wei, Junliang Xing, and Shiming Xiang.
\newblock Dual critic reinforcement learning under partial observability.
\newblock In \emph{The Thirty-eighth Annual Conference on Neural Information Processing Systems}, 2024.
\newblock URL \url{https://openreview.net/forum?id=GruuYVTGXV}.

\bibitem[Ni et~al.(2022)Ni, Eysenbach, Levine, and Salakhutdinov]{expt3}
Tianwei Ni, Benjamin Eysenbach, Sergey Levine, and Ruslan Salakhutdinov.
\newblock Recurrent model-free {RL} is a strong baseline for many {POMDP}s, 2022.
\newblock URL \url{https://openreview.net/forum?id=E0zOKxQsZhN}.

\bibitem[Aberdeen et~al.(2007)Aberdeen, Buffet, and Thomas]{expt4}
Douglas Aberdeen, Olivier Buffet, and Owen Thomas.
\newblock Policy-gradients for psrs and pomdps.
\newblock In Marina Meila and Xiaotong Shen, editors, \emph{Proceedings of the Eleventh International Conference on Artificial Intelligence and Statistics}, volume~2 of \emph{Proceedings of Machine Learning Research}, pages 3--10, San Juan, Puerto Rico, 21--24 Mar 2007. PMLR.
\newblock URL \url{https://proceedings.mlr.press/v2/aberdeen07a.html}.

\end{thebibliography}
\newpage
\appendix
\section{Appendix A: Omitted Proofs}
\subsection{Proof of Lemma \ref{lemma: belief approx}}
Let $H$ and $H'$ be two histories with the same superstate, $\mathcal{G}(H')=\mathcal{G}(H) = \{a_{1},y_{1},a_{2},\ldots,y_{(l-1)},a_{l},y_{l}\}$. Moreover, let $t_1$ and $t_2$ be lengths of $H$ and $H'$, respectively. Then, we have 
\begin{align*}
&\boldsymbol{\pi}(H) = K_{a_{l},y_{l}} \otimes\ldots \otimes K_{a_{1},y_{1}} \otimes \boldsymbol{\pi}(H_{t_1-l}),\cr
&\boldsymbol{\pi}(H') = K_{a_{l},y_{l}} \otimes\ldots \otimes K_{a_{1},y_{1}} \otimes \boldsymbol{\pi}(H'_{t_2-l}).
\end{align*}
Using Assumption \ref{as:UFSC} inductively, we can write
\begin{align}\nonumber
    \| \boldsymbol{\pi}(H) - \boldsymbol{\pi}(H') \|_{TV} &= \| K_{a_{l},y_{l}} \otimes\ldots \otimes K_{a_{1},y_{1}} \otimes \boldsymbol{\pi}(H_{t_1-l}) -  K_{a_{l},y_{l}} \otimes\ldots \otimes K_{a_{1},y_{1}} \otimes \boldsymbol{\pi}(H'_{t_2-l})\|_{TV} \cr
    &\leq (1-\rho)^l \| \boldsymbol{\pi}(H_{t_1-l}) - \boldsymbol{\pi}(H'_{t_2-l})  \|_{TV} \cr
    &\leq (1-\rho)^l. 
\end{align}

\subsection{Proof of Lemma \ref{lemma: norm bound}}

Without loss of generality, let us assume $\sum_i a_i b_i - \sum_i c_i d_i \ge  0$. We have
\begin{align*}
            \sum_i a_i b_i &- \sum_i c_i d_i \nonumber \\
            &= \sum_i \Big[\frac{(a_i-c_i)(b_i+d_i)}{2} + \frac{(b_i-d_i)(a_i+c_i)}{2}\Big]\cr
            &\leq \sum_i \Big[\frac{(a_i-c_i)(b_i+d_i)}{2} + \frac{ \mid  b_i-d_i \mid  (a_i+c_i)}{2}\Big] \cr
            &= \sum_{i:a_i \ge  c_i}\bigg[(a_i-c_i)\Big(\max(b_i,d_i)-\frac{ \mid  b_i-d_i \mid  }{2}\Big) +   \mid  b_i-d_i \mid  \Big(\frac{a_i+c_i}{2}\Big) \bigg]\cr
            &+ \sum_{i:c_i > a_i}\bigg[(a_i-c_i)\Big(\frac{b_i+d_i}{2}\Big) +  \mid  b_i-d_i \mid  \Big(c_i-\frac{(c_i-a_i)}{2}\Big)\bigg]  \cr
            &\leq \sum_{i:a_i \ge  c_i} \bigg[(a_i-c_i)\Big(\max(b_i,d_i)-\frac{ \mid  b_i-d_i \mid  }{2}\Big)+   \mid  b_i-d_i \mid  \Big(\frac{a_i+c_i}{2}\Big)\bigg] \cr
            &+  \sum_{i: c_i>a_i }  \mid  b_i-d_i \mid  \Big(c_i-\frac{(c_i-a_i)}{2}\Big)\cr
            &\leq \sum_{i:a_i \ge  c_i} \Big[(a_i-c_i)\max(b_i,d_i) +  c_i \mid  b_i-d_i \mid    \Big]+ \sum_{i:c_i>a_i}  \mid  b_i-d_i \mid  \Big(c_i-\frac{(c_i-a_i)}{2}\Big) \cr
            &\leq \sum_{i:a_i \ge  c_i} \Big[(a_i-c_i)\max( \|  b \|  _{\infty}, \|  d \|  _{\infty}) +  c_i\|  b-d \|  _{\infty} \Big]+ \sum_{i:c_i>a_i }  \|  b-d \|  _{\infty} \Big(c_i-\frac{(c_i-a_i)}{2}\Big) \cr
            &= \sum_{i:a_i \ge c_i} (a_i-c_i)\max( \|  b \|  _{\infty}, \|  d \|  _{\infty})  - \sum_{i: c_i>a_i}  \|  b-d \|  _{\infty} \Big(\frac{c_i-a_i}{2}\Big)+  \|  b-d \|  _{\infty} \nonumber \\
            &= \frac{ \|  a-c \|_1}{2} \max( \|  b \|  _{\infty}, \|  d \|_{\infty}) - \|  b-d \|  _{\infty} \frac{ \|  a-c \|_1}{4}   +  \|  b-d \|_{\infty},
            \end{align*} 
        where the last equality follows from relations $\sum_{i: a_i \ge  c_i} (a_i - c_i) + \sum_{i: c_i>a_i } (c_i - a_i) =  \|  a-c \|  _1$ and $\sum_{i: a_i \ge  c_i} (a_i - c_i) - \sum_{i: c_i>a_i } (c_i - a_i) = \sum_i a_i - \sum_i c_i = 0$, which together implies that $\sum_{i: a_i \ge c_i} (a_i - c_i) = \sum_{i: c_i>a_i } (c_i - a_i) = \frac{ \|  a-c \|  _1}{2}.$

\subsection{Proof of Theorem \ref{thm: main result}}

Consider an arbitrary history $H \in \mathbb{H}$ with the corresponding superstate $\mathcal{G}(H) = B$. For any $a\in \mathcal{A}$,
\begin{align}\label{eq: reward bound}
     \Big|  \tilde{r}(B,a) - \sum_s \pi(s \mid H)r(s,a) \Big|  &=  \Big|  \sum_s \pi(s \mid  B)r(s,a) -  \sum_s \pi(s \mid  H)r(s,a) \Big|\cr
     &\leq \sum_s \big|  \pi(s \mid  B) -  \pi(s \mid  H) \big||r(s,a)| \cr
     &\leq 2(1-\rho)^l \bar{r},
\end{align}
where the last inequality follows from Lemma \ref{lemma: belief approx}.

Let $ \delta:=\|   V^*\big(\boldsymbol{\pi}(H) \big) - \tilde{V}\big(\mathcal{G}(H) \big) \|  _{\infty}$, and note that $\delta$ is finite since the value functions are finite.
\begin{equation}
    \delta =\|   V^*\big(\boldsymbol{\pi}(H) \big) - \tilde{V}\big(\mathcal{G}(H) \big) \|  _{\infty}=\|   V^*\big(\boldsymbol{\pi}(H) \big) - \tilde{V}\big(B \big) \|  _{\infty}
\end{equation}
Next, we bound $|   V^*\big(\boldsymbol{\pi}(H) \big) - \tilde{V}\big(\mathcal{G}(H) \big) |$ for all $H$.
Using Bellman's Optimality equation for the belief state MDP (POMDP) and superstate MDP, we have 
\begin{align}\nonumber
     &|   V^*\big(\boldsymbol{\pi}(H) \big) - \tilde{V}\big(\mathcal{G}(H) \big) | \cr
     &= \max_a \Big\{ \sum_s \pi(s \mid  H) r(s,a) + \gamma \sum_y V^*\big(\boldsymbol{\pi}(H \,\|\, \{y,a\})\big) \sigma(\boldsymbol{\pi}(H),y,a)\Big\} \cr
    &- \max_a \Big\{ \tilde{r}(B,a) + \gamma \sum_y \tilde{V}\big(\mathcal{G}\big(B \,\|\, \{y,a\})\big)\sigma(\boldsymbol{\pi}(B),y,a)\Big\}.
\end{align}
Suppose $\hat{a}$ is the best action for the POMDP at belief state $\boldsymbol{\pi}(H)$. Then we can write $|   V^*\big(\boldsymbol{\pi}(H) \big) - \tilde{V}\big(\mathcal{G}(H) \big) |$ as
\begin{align}
    &|   V^*\big(\boldsymbol{\pi}(H) \big) - \tilde{V}\big(\mathcal{G}(H) \big) | \cr
    &=  \sum_s \pi(s \mid  H) r(s,\hat{a}) + \gamma \sum_y V^*\big(\boldsymbol{\pi}(H \,\|\, \{y,\hat{a}\})\big) \sigma(\boldsymbol{\pi}(H),y,\hat{a}) \cr
    &- \max_a \Big\{ \tilde{r}(B,a) + \gamma \sum_y \tilde{V}\big(\mathcal{G}\big(B \,\|\, \{y,a\})\big)\sigma(\boldsymbol{\pi}(B),y,a)\Big\}.
\end{align}
Now, since the maximum value of the second term will be greater than evaluating the second term for $a = \hat{a}$, we have 
\begin{align}\nonumber
    & |   V^*\big(\boldsymbol{\pi}(H) \big) - \tilde{V}\big(\mathcal{G}(H) \big) | \leq \Big(\sum_s \pi(s \mid  H) r(s,\hat{a})-\tilde{r}(B,\hat{a})\Big)\cr
    &+\gamma \Big[\sum_yV^*\big(\boldsymbol{\pi}(H \,\|\, \{y,\hat{a}\})\big) \sigma(\boldsymbol{\pi}(H),y,\hat{a})- \sum_y\tilde{V}\big(\mathcal{G}(B \,\|\, \{y,\hat{a}\})\big) \sigma(\boldsymbol{\pi}(B),y,\hat{a}) \Big]\cr
    &\leq 2(1-\rho)^l\bar{r} + \gamma  \Big[\sum_y V^*\big(\boldsymbol{\pi}(H \,\|\, \{y,\hat{a}\})\big) \sigma(\boldsymbol{\pi}(H),y,\hat{a})- \sum_y \tilde{V}\big(\mathcal{G}(B \,\|\, \{y,\hat{a}\})\big) \sigma(\boldsymbol{\pi}(B),y,\hat{a}) \Big], 
    \end{align}
    where in the second inequality we have used \eqref{eq: reward bound}. 
    
    Next, we note that $\sum_y\sigma(\boldsymbol{\pi}(H),y,\hat{a})=\sum_y\sigma(\boldsymbol{\pi}(B),y,\hat{a})=1$, as $\sigma(\cdot)$ is a probability distribution over the observation set $\mathcal{Y}$. Moreover, by the definition of $\delta$ and since $\max\big(\|\tilde{V}\|_{\infty}, \|V^*\|_{\infty}\big)\leq \frac{\bar{r}}{1-\gamma}$, we can use Lemma \ref{lemma: norm bound} to upper-bound the second term in the above expression and obtain

\begin{align}
\label{Eq: delta bound}
  &|   V^*\big(\boldsymbol{\pi}(H) \big) - \tilde{V}\big(\mathcal{G}(H) \big) | \leq 2(1-\rho)^l\bar{r} \cr
  &+ \gamma  \bigg[\big\|   \sigma(\boldsymbol{\pi}(H),\cdot,\hat{a})- \sigma(\boldsymbol{\pi}(B),\cdot,\hat{a}) \big\|  _{TV} \frac{\bar{r}}{1-\gamma}+\delta-\frac{\delta}{2}\big\|   \sigma(\boldsymbol{\pi}(H),\cdot,\hat{a}) - \sigma(\boldsymbol{\pi}(B),\cdot,\hat{a}) \big\|  _{TV}\bigg]\cr
  &=  2(1-\rho)^l\bar{r} + \gamma \bigg[ \big\|   \sigma(\boldsymbol{\pi}(H),\cdot,\hat{a})- \sigma(\boldsymbol{\pi}(B),\cdot,\hat{a}) \big\|  _{TV}\bigg(\frac{\bar{r}}{1-\gamma} - \frac{\delta}{2}\bigg) + \delta\bigg].
\end{align}
Note that the above inequality holds for all $H.$

Furthermore, we can bound the total variation norm in the above relation as  
\begin{align}\nonumber
\big\|   \sigma(\boldsymbol{\pi}(H),\cdot,\hat{a})- \sigma(\boldsymbol{\pi}(B),\cdot,\hat{a}) \big\|  _{TV}&=\sum_y \big|   \sigma(\hat{a},y,\boldsymbol{\pi}(H)) - \sigma(\hat{a},y,\boldsymbol{\pi}(B)) \big| \cr
&= \sum_y  \big|  \sum_{s,s'}\Phi(y \mid  s') \mathcal{P}(s' \mid  s,\hat{a}) (\pi(s \mid  H) - \pi(s \mid  B))  \big| \cr
&\leq  \sum_s \big| \pi(s \mid  H) - \pi(s \mid  B)  \big| \sum_{y,s'}\Phi(y \mid  s') \mathcal{P}(s' \mid  s,\hat{a}) \cr
&=2\|\boldsymbol{\pi}(H)-\boldsymbol{\pi}(B)\|_{TV}
\leq 2(1-\rho)^l,
\end{align}
where the first inequality holds by the triangle inequality, and the second inequality is obtained using Lemma \ref{lemma: belief approx}. 
Now, we consider two cases. 
Case 1: $\frac{\bar{r}}{1-\gamma} - \frac{\delta}{2} \geq 0$

Therefore, we have
\begin{equation}
    \delta \leq  2(1-\rho)^l\bar{r} + \gamma \bigg[ \sup_{H} \big\|   \sigma(\boldsymbol{\pi}(H),\cdot,\hat{a})- \sigma(\boldsymbol{\pi}(B),\cdot,\hat{a}) \big\|  _{TV}\bigg(\frac{\bar{r}}{1-\gamma} - \frac{\delta}{2}\bigg) + \delta\bigg]
\end{equation}

Case 2: $\frac{\bar{r}}{1-\gamma} - \frac{\delta}{2} < 0$

Therefore, we have
\begin{equation}
    \delta \leq  2(1-\rho)^l\bar{r} + \gamma \bigg[ \inf_{H} \big\|   \sigma(\boldsymbol{\pi}(H),\cdot,\hat{a})- \sigma(\boldsymbol{\pi}(B),\cdot,\hat{a}) \big\|  _{TV}\bigg(\frac{\bar{r}}{1-\gamma} - \frac{\delta}{2}\bigg) + \delta\bigg]
\end{equation}

Now, both $\inf_{H} \big\|   \sigma(\boldsymbol{\pi}(H),\cdot,\hat{a})- \sigma(\boldsymbol{\pi}(B),\cdot,\hat{a}) \big\|  _{TV}$ and $\sup_{H} \big\|   \sigma(\boldsymbol{\pi}(H),\cdot,\hat{a})- \sigma(\boldsymbol{\pi}(B),\cdot,\hat{a}) \big\|  _{TV}$ are less than $2(1-\rho)^l$.

Therefore, we can write
\begin{align}
    \delta &\leq 2(1-\rho)^l\bar{r} + \gamma \bigg[ 2 (1-\rho)^l \bigg(\frac{\bar{r}}{1-\gamma} - \frac{\delta}{2}\bigg) + \delta\bigg] \nonumber \\
    &\leq \frac{ 2(1-\rho)^l\bar{r} }{1-\gamma} + \frac{\gamma}{1-\gamma} \bigg[ 2(1-\rho)^l\bigg(\frac{\bar{r}}{1-\gamma} - \frac{\delta}{2}\bigg)\bigg] \nonumber \\
    \delta \Big(1+\frac{\gamma (1-\rho)^l}{1-\gamma}\Big)&\leq \frac{ 2(1-\rho)^l\bar{r} }{1-\gamma}  + \frac{2\gamma(1-\rho)^l\bar{r}}{(1-\gamma)^2} \nonumber \\
   \delta & \leq\frac{2(1-\rho)^l\bar{r}}{(1-\gamma) (1+\frac{\gamma (1-\rho)^l}{1-\gamma}) } + \frac{2\gamma\bar{r}(1-\rho)^l}{(1-\gamma)(1-\gamma + \gamma(1-\rho)^l) } \nonumber \\
    &\leq \frac{2(1-\rho)^l\bar{r}}{(1-\gamma)} + \frac{2\gamma\bar{r}(1-\rho)^l}{(1-\gamma)(1-\gamma + \gamma(1-\rho)^l) }
\end{align}
\subsection{ Corollary \ref{Corollary: state comparison}}
Additionally, if $N$ is the number of states of the Superstate MDP. Then, we can state the following result:
\begin{corollary}
\label{Corollary: state comparison}
If $N$ is the number of states in the Superstate MDP, then the difference of the optimal value functions in Theorem \ref{thm: main result} can be upper-bounded as
    \begin{align*}
      \big\|   V^{\mu^*}\big(\boldsymbol{\pi}(H) \big) - \tilde{V}\big(\mathcal{G}(H) \big) \big\|  _{\infty}  
     \leq \frac{ 2\bar{r} (1-\rho)^{-1} N^{-\kappa} }{1-\gamma}  + \frac{ 4 \bar{r}\gamma (1-\rho)^{-1} N^{-\kappa} }{(1-\gamma)\big(2(1-\gamma) + \gamma N^{-\kappa} \big)}, 
    \end{align*}
    where $\kappa = \frac{\log (1/(1-\rho))}{\log ( |\mathcal{Y}| |\mathcal{A}| )}.$
\end{corollary}
\textbf{Proof:}\\
    For a fixed $l$, the number of superstates can be all possible histories of length at most $l$. Thus, 
    \begin{align}\nonumber
        N = 1 + |\mathcal{Y}||\mathcal{A}| + |\mathcal{Y}|^2 |\mathcal{A}|^2 + \ldots +|\mathcal{Y}|^l |\mathcal{A}|^{l} < |\mathcal{Y}|^{l+1} |\mathcal{A}|^{l+1},
    \end{align}
    which implies $l+1 > \frac{\log N}{\log ( \mid  \mathcal{Y} \mid \mid  \mathcal{A} \mid  )}$. Therefore, $(1-\rho)^l < (1-\rho)^{-1} \cdot N^{\frac{\log (1-\rho)}{\log ( \mid  \mathcal{Y} \mid  \mid  \mathcal{A} \mid  )}}$. Similarly, since $l < \frac{\log N}{\log ( \mid  \mathcal{Y} \mid   \mid  \mathcal{A} \mid  )}$, we have $(1-\rho)^l >  N^{\frac{\log (1-\rho)}{\log ( \mid  \mathcal{Y} \mid  \mid  \mathcal{A} \mid  )}}$. Substituting these relations into Eq. \eqref{eq: value fn bound} gives us the desired result.
   
\subsection{Proof of Lemma \ref{lemma: mixing}}

Let $d_1, d_2$ be distributions over the superstates and similarly $e_1, e_2$ be the canonical basis vectors in $\mathbb{R}^{|\mathbb{H}_{\leq l}}$. Then there exists $\alpha_{i,j}$ such that $\alpha_{i,j} \geq 0$ and $\sum_{i,j}\alpha_{i,j} = \| d_1-d_2\|_{TV}$ and the following holds 
\begin{align}
    \big\| \tilde{\mathcal{P}}^{\mu}d_1- \tilde{\mathcal{P}}^{\mu}d_2 \big\|_{TV} &= \| \tilde{\mathcal{P}}^{\mu}(d_1-d_2)\|_{TV} \nonumber \\
    &= \| \sum_{i,j} \alpha_{i,j} \tilde{\mathcal{P}}^{\mu}(e_i-e_j)\|_{TV} \nonumber \\
    &\leq \sum_{i,j} \alpha_{i,j} \| \tilde{\mathcal{P}}^{\mu}(e_i-e_j)\|_{TV}
\end{align}
Note that the maximum value of the above term is $\sum_{i,j}$ which is $\| d_1-d_2\|_{TV}$. Therefore, if for some $i,j$ we have $\alpha_{i,j} > 0$ and $\| \tilde{\mathcal{P}}^{\mu}(e_i-e_j)\|_{TV} < 1$, we are guaranteed a contraction.

Now suppose that $e_i$ corresponds to a superstate $B_i$ and $e_j$ corresponds to a superstate $B_j$
\begin{align}
    &\| \tilde{\mathcal{P}}^{\mu}(e_i-e_j)\|_{TV} \nonumber \\
    = 
    &\| \sum_a \mu(a\mid B_i)\sum_{y,s,s'} \mathbb{I}[\mathcal{G}\big(B_i \,\|\, \{y,a\}\big) = B']\cdot \Phi(y\mid s')\mathcal{P}(s' \mid s,a) \pi(s \mid B_i) \nonumber \\
    &- \sum_a \mu(a\mid B_j)\sum_{y,s,s'} \mathbb{I}[\mathcal{G}\big(B_j \,\|\, \{y,a\}\big) = B']\cdot \Phi(y\mid s')\mathcal{P}(s' \mid s,a) \pi(s \mid B_j)\|_{TV}
\end{align}
Next, we will show that if $B_i$ and $ B_j$ are two superstates which differ in the first two elements then $\| \tilde{\mathcal{P}}^{\mu}(B_i-B_j)\|_{TV} < 1$. 

Therefore, we focus on pairs of $B_i, B_j$ such that $\mathcal{G}\big(B_i \,\|\, \{y,a\}\big) = \mathcal{G}\big(B_j \,\|\, \{y,a\}\big)$, i.e., $B_i, B_j$ which only differ in the first two elements. For such a pair, we can simplify further
\begin{align}
    &\| \sum_a \mu(a\mid B_i)\sum_{y,s,s'} \mathbb{I}[\mathcal{G}\big(B_i \,\|\, \{y,a\}\big) = B']\cdot \Phi(y\mid s')\mathcal{P}(s' \mid s,a) \pi(s \mid B_i) \nonumber \\
    &- \sum_a \mu(a\mid B_j)\sum_{y,s,s'} \mathbb{I}[\mathcal{G}\big(B_j \,\|\, \{y,a\}\big) = B']\cdot \Phi(y\mid s')\mathcal{P}(s' \mid s,a) \pi(s \mid B_j)\|_{TV} \nonumber \\
    &= \|  \{\mu(a\mid B_i)\sum_{s,s'}  \Phi(y\mid s')\mathcal{P}(s' \mid s,a) \pi(s \mid B_i) - \mu(a\mid B_j)\sum_{s,s'} \Phi(y\mid s')\mathcal{P}(s' \mid s,a) \pi(s \mid B_j)\}_{y,a}\|_{TV} \nonumber \\
    &= 1/2\sum_{y,a} |\mu(a\mid B_i)\sum_{s,s'}  \Phi(y\mid s')\mathcal{P}(s' \mid s,a) \pi(s \mid B_i) - \mu(a\mid B_j)\sum_{s,s'} \Phi(y\mid s')\mathcal{P}(s' \mid s,a) \pi(s \mid B_j) | \nonumber \\
    &\leq 1/2\sum_{y,a} \Bigg[\Big|\mu(a\mid B_j)\Big[\sum_{s,s'}  \Phi(y\mid s')\mathcal{P}(s' \mid s,a) \pi(s \mid B_i) - \sum_{s,s'} \Phi(y\mid s')\mathcal{P}(s' \mid s,a) \pi(s \mid B_j)\Big] \Big| \nonumber \\
    &+ \Big | (\mu(a\mid B_i) - \mu(a\mid B_j)) \sum_{s,s'}  \Phi(y\mid s')\mathcal{P}(s' \mid s,a) \pi(s \mid B_i) \Big |\Bigg] \nonumber \\
    &\leq \| \pi(s \mid B_i) - \pi(s \mid B_j) \|_{TV} + \|\mu(a \mid B_i) - \mu(a \mid B_j)\|_{TV}  \nonumber \\
    &\leq (1-\rho)^l + 1-|\mathbb{A}|\delta \nonumber \\
\end{align}
For the last inequality we assume that our policy has a small exploration component such that $\mu(a\mid B) \geq \delta \ \ \forall \ B$. Therefore, for sufficiently large horizon length $l$ or exploration $\delta$, assumption 2 is automatically satisfied.

Therefore if $\alpha_{i,j}\ > 0$ for pairs of $B_i, B_j$ such that $\mathcal{G}\big(B_i \,\|\, \{y,a\}\big) = \mathcal{G}\big(B_j \,\|\, \{y,a\}\big)$ we are guaranteed a contraction

Next, we will show that we can construct an algorithm such that $\alpha_{i,j} > 0$ for all such pairs $B_i, B_j$.

To construct $\alpha_{i,j}$ the following greedy algorithm can be used.
Let \( v_1, v_2 \in \mathbb{R}^n \) be two probability distributions. Define the difference vector \( d = v_1 - v_2 \), and define:

\begin{algorithm}\caption{ A Greedy Algorithm to construct $\alpha$}
\begin{algorithmic}
\State\textbf{Input:} Two discrete distributions \( v_1 = (v_1(1), v_1(2), \dots, v_1(n)) \) and \( v_2 = (v_2(1), v_2(2), \dots, v_2(n)) \) 
\State Compute the difference vector \( \Delta = v_1 - v_2 \) where \( \Delta_i = v_1(i) - v_2(i) \)
 \State Define the surplus set \( S = \{ i \mid \Delta_i > 0 \} \) and the deficit set \( D = \{ j \mid \Delta_j < 0 \} \)
\State Initialize \( \alpha_{ij} = 0 \) for all \( i, j \) 
\State \textbf{For each} \( i \in S \) \textbf{do}
  \State \hspace{1em} \textbf{For each} \( j \in D \) \textbf{do}
    \State \hspace{2em}  \( \alpha_{ij} \gets \min(\Delta_i, -\Delta_j) \)
\State \hspace{2em}           \( \Delta_i \gets \Delta_i - \alpha_{ij} \)
\State \hspace{2em}           \( \Delta_j \gets \Delta_j + \alpha_{ij} \)
      \State \hspace{1em}      \textbf{end}
      \State      \textbf{end}
      \end{algorithmic}
\end{algorithm}
See that first $\alpha_{ij} = \min(\Delta_i, -\Delta_j)$. Therefore, since $\Delta_i, \Delta_j \neq 0$, $\alpha_{ij} > 0$. Additionally, the steps $\Delta_i \gets \Delta_i - \alpha_{ij}$, $\Delta_j \gets \Delta_j + \alpha_{ij} $ ensures that either one of them goes to 0 and is removed from the surplus/deficit set.

Additionally, it is straightforward to see that $\sum_i = \|d\|$. Thus, $\Delta_i, \Delta_j \neq 0 \implies \alpha_{ij} > 0$ 
    
\subsection{Proof of Lemma \ref{Lemma: Approx TD}}
We first introduce some notations that will be used to prove the result. Let us define 
\[g_t(\theta) := \big[r_t + \gamma \phi^T\big(\mathcal{G}(H_{t+1}),a_{t+1}\big)\theta - \phi^T\big(\mathcal{G}(H_t),a_t\big)\theta\big]\cdot\phi(\mathcal{G}\big(H_t),a_t\big),\] 
and note that it can be written in a compact form as $g_t(\theta) = \Phi R_t + \gamma \Phi E_t \Phi^T \theta - \Phi D_t \Phi^T \theta$, where\footnote{$[v]_{v \in V}$ denotes a matrix obtained by concatenating all vectors in $V$, where each $v$ is a column vector of the matrix.}
\begin{align*}
&D_t := \mathrm{diag}\left(\big[\mathbb{I}[(B_t,a_t) = (B,a)]\big]_{B \in \mathbb{H}_{\leq l}, a \in \mathcal{A}}\right),\cr
&R_t := \big[r_t \mathbb{I}[(B_t,a_t) = (B,a)]\big]^T_{B \in \mathbb{H}_{\leq l}, a \in \mathcal{A}}, \cr
&E_t\big((B,a),(B',a')\big) := \mathbb{I}[(B_t,a_t) = (B,a)] \cdot \mathbb{I}[(B_{t+1},a_{t+1}) = (B',a')],\cr
&\Phi = [\phi(B,a)]_{B \in \mathbb{H}_{\leq l}, a \in \mathcal{A}}.
\end{align*}
Additionally, let $\bar{g}(\theta) := \Phi \tilde{D}^{\mu} \tilde{r} + \gamma \Phi \tilde{D}^{\mu} \tilde{P}^{\mu} \Phi^T \theta - \Phi \tilde{D}^{\mu} \Phi^T \theta$, where $\tilde{D}^{\mu}$ and $\tilde{P}^{\mu}$ denote the stationary distribution and the state transition matrix for the Superstate MDP when following policy $\mu$, respectively. In particular, $\tilde{P}^{\mu}(B',a'\mid B,a) =\mu(a'\mid B) \tilde{\mathcal{P}}(B'\mid B,a)$, where $\tilde{\mathcal{P}}(B'\mid B,a)$ is given by \eqref{Eq: superstate transition prob}. Similarly $\tilde{r} := [r(B,a)]^T_{B \in \mathbb{H}_{\leq l}, a \in \mathcal{A}}$ is the vector of rewards corresponding to the Superstate MDP as defined in \eqref{Eq: superstate reward}. Finally, we also define $$\eta_t(\theta) = (\theta-\hat{\theta})^T (g_t(\theta) - \bar{g}(\theta)).$$

Next, we state and prove an auxiliary lemma that would be required for our main analysis.

\begin{lemma}\label{lemma: res1}
The following inequalities are true:\footnote{Unless stated, the norms are $L_2$-norms.}
\begin{itemize}
\item[(a)] $|\eta_t(\theta_t)| \leq C_1$, where $C_1=2R \cdot 2(\bar{r} + 2R).$  
\item[(b)] $\| g_t(\theta_1) - g_t(\theta_2) \| \leq C_2 \|\theta_1-\theta_2\| \ \forall \theta_1,\theta_2$, where $C_2=(2\bar{r} + 12R)$. 
\end{itemize}
\end{lemma}

\textbf{Proof:}\\
To show part (a), using the Cauchy-Schwarz inequality, we have
\begin{align}\nonumber
    |\eta_t(\theta_t)| &\leq \|(\theta_t-\hat{\theta})\| \|(g_t(\theta_t) - \bar{g}(\theta_t))\|\leq 2R \cdot 2(\bar{r} + 2R).
\end{align}
To prove part (b), for simplicity let $B=\mathcal{G}(H_t), a_t=a$, and $B'=\mathcal{G}(H_{t+1}), a_{t+1}=a'$. Then,
        \begin{align}\nonumber
    \| g_t(\theta_1) - g_t(\theta_2) \| &= \| (\gamma \phi^T (B',a') - \phi^T(B,a))(\theta_1-\theta_2) \phi(B,a) \| \cr
    &\leq \|\theta_1-\theta_2\| \|\phi(B,a)\| \|\gamma \phi^T (B',a') - \phi^T(B,a)\|\leq 2 \|\theta_1-\theta_2\|
\end{align}
Similarly, we can show that $\| \bar{g}(\theta_1) - \bar{g}(\theta_2) \| \leq 2 \|\theta_1-\theta_2\|.$ Thus, 
\begin{align}\nonumber
    |\eta_t(\theta_1) - \eta_t(\theta_2)|
    &= |\big(g_t(\theta_1)-\bar{g}(\theta_1)\big)^T(\theta_1-\theta_2 + \theta_2 - \hat{\theta}) - \big((g_t(\theta_2) - \bar{g}(\theta_2))\big)^T (\theta_2-\hat{\theta}) | \cr
    &\leq \| g_t(\theta_1)-\bar{g}(\theta_1)\| \| \theta_1-\theta_2\| + \|\theta_2-\hat{\theta}\| (\|g_t(\theta_1)-g(\theta_2)\| + \| \bar{g}(\theta_1) - \bar{g}(\theta_2) \|) \cr
    &\leq (2\bar{r} + 12R) \|\theta_1-\theta_2\|. 
\end{align}

We are now ready to prove Lemma \ref{Lemma: Approx TD}. Let us consider the Lyapunov function  
\begin{align}\nonumber
\mathcal{L}(\theta) := \|\theta-\hat{\theta}\|^2. 
\end{align} 
In order to show that $\theta_t$ converges to $\hat{\theta}$, we will show that $\mathcal{L}(\theta_t)$ converges to $0$, and obtain finite time bounds for the convergence. To that end, we first relate the successive iterates $\mathcal{L}(\theta_t)$ and $\mathcal{L}(\theta_{t+1})$:
\begin{align}
      \mathcal{L}(\theta_{t+1}) &= \|\theta_{t+1}-\hat{\theta}\|^2 \nonumber \\
      &= \| \mathrm{Proj}(\theta_{t+1/2}) - \mathrm{Proj}(\hat{\theta}) \|^2 \nonumber \\
      &\leq \| \theta_{t+1/2}-\hat{\theta} \|^2 \nonumber \\
      &= \| \theta_t + \epsilon_t g_t(\theta_t) - \hat{\theta} \|^2 \nonumber \\
      &= \|\theta_t - \hat{\theta}\|^2 + \epsilon_t^2 \|g_t(\theta_t)\|^2 +2\epsilon_t g_t^T(\theta_t)(\theta_t-\hat{\theta}) \nonumber \\
      \label{eq: lyapunov bound1}
      &\leq \epsilon_t^2 (\bar{r} + 2R)^2 + \mathcal{L}(\theta_t) + 2 \epsilon_t g_t^T(\theta_t)(\theta_t-\hat{\theta}),
    \end{align}
    where the last step follows from $\|g_t(\theta_t)\| \leq \bar{r} + 2R$. By adding and subtracting $\bar{g}(\theta_t)$ in Eq. \eqref{eq: lyapunov bound1}, we get
\begin{align}\nonumber
    \mathcal{L}(\theta_{t+1}) \leq \mathcal{L}(\theta_t) +\epsilon_t^2 (\bar{r} + 2R)^2 + 2 \epsilon_t (\theta_t-\hat{\theta})^T \bar{g}(\theta_t) +2 \epsilon_t (\theta_t-\hat{\theta})^T (g_t(\theta_t) - \bar{g}(\theta_t)).
\end{align}

Next, we proceed to bound each of the terms in the above expression. We start by bounding $2 \epsilon_t (\theta_t-\hat{\theta})^T \bar{g}_t(\theta_t)$ in terms of $\mathcal{L}(\theta_t)$. We can write
\begin{align}\nonumber
    &2 \epsilon_t (\theta_t-\hat{\theta})^T \bar{g}_t(\theta_t) \cr
    &= 2 \epsilon_t (\theta_t - \hat{\theta})^T \Phi (\tilde{D}^{\mu}\tilde{r} + \gamma \tilde{D}^{\mu} \tilde{P}^{\mu} \Phi^T \theta_t - \tilde{D}^{\mu} \Phi^T \theta_t) \cr
    &=  2 \epsilon_t \big(\Phi^T(\theta_t - \hat{\theta})\big)^T \tilde{D}^{\mu} (\tilde{T}^{\mu}(\Phi^T \theta_t) - \tilde{T}^{\mu}(\Phi^T \hat{\theta}) ) + 2 \epsilon_t \big(\Phi^T(\theta_t - \hat{\theta})\big)^T \tilde{D}^{\mu} (\Phi^T \hat{\theta} - \Phi^T \theta_t) \cr
    &\leq 2\epsilon_t \| \Phi^T (\theta_t -\hat{\theta}) \|_{\tilde{D}^{\mu}} \| \tilde{T}^{\mu}(\Phi^T\theta_t) -\tilde{T}^{\mu}(\Phi^T\hat{\theta}) \|_{\tilde{D}^{\mu}}   -2 \epsilon_t \| \Phi^T (\theta_t -\hat{\theta}) \|_{\tilde{D}^{\mu}}^2 \cr
    &\leq 2\epsilon_t (\gamma-1) \| \Phi^T (\theta_t -\hat{\theta}) \|_{\tilde{D}^{\mu}}^2,
\end{align}
where the first inequality uses the Cauchy-Schwarz inequality and the second inequality follows from the contraction property.

Next, we will proceed to bound $\mathbb{E}[2 \epsilon_t (\theta_t-\hat{\theta})^T (g_t(\theta_t) - \bar{g}(\theta_t))]$. To that end, we will first relate $\eta_t(\theta_t)$ with $\eta_t(\theta_{t-l'})$. Note that since
\begin{align}\nonumber
    \| \theta_{t+1} - \theta_t \| 
    &= \|\textrm{Proj}(\theta_t-\epsilon_t g_t(\theta_t)) - \textrm{Proj}(\theta_t) \| \cr
    &\leq \| \epsilon_t g_t(\theta_t) \| = (\bar{r} + 2R) \epsilon_t,
\end{align}
we have $\| \theta_t - \theta_{t-l'} \| \leq (\bar{r} + 2R) \sum_{i=t-l'}^{t-1} \epsilon_i$. Therefore, using Lemma \ref{lemma: res1} (part b), we obtain $$\eta_t(\theta_t) \leq \eta_t(\theta_{t-l'}) + C_2 (\bar{r} + 2R) \sum_{i=t-l'}^{t-1} \epsilon_i.$$

Let $\mathcal{F}_{t-l'} = \{y_0,a_0,r_0,y_1,\ldots,y_t,a_t,r_t,y_{t+1},a_{t+1}\}$ be the filtration up to time $t-l'$, such that conditioned on $\mathcal{F}_{t-l'}$, $\theta_{t-l'}$ is measurable and  deterministic. We can now obtain an upper bound on $\mathbb{E}[\eta_t(\theta_{t-l'})]$ as follows:
\begin{align}\label{eq:eta-theta}
  \mathbb{E}[\eta_t(\theta_{t-l'})]
  &= \mathbb{E}\big[\mathbb{E}[\eta_t(\theta_{t-l'}) \mid \mathcal{F}_{t-l'}]\big] \nonumber \\
  &= \mathbb{E}\big[\mathbb{E}\big[\big(\Phi^T(\theta_{t-l'}-\hat{\theta})\big)^T (R_t - \tilde{D}^{\mu} \tilde{r}) \mid \mathcal{F}_{t-l'} \big]\big] \nonumber \\  &+\mathbb{E}\big[\mathbb{E}\big[\big(\Phi^T(\theta_{t-l'}-\hat{\theta})\big)^T (\gamma E_t \Phi^T \theta_{t-l'}- \gamma \tilde{D}^{\mu} \tilde{P}^{\mu} \Phi^T \theta_{t-l'}) \mid \mathcal{F}_{t-l'} \big]\big] \nonumber \\
  &+\mathbb{E}\big[\mathbb{E}\big[\big(\Phi^T(\theta_{t-l'}-\hat{\theta})\big)^T (\tilde{D}^{\mu} \Phi^T \theta_{t-l'} - D_t \Phi^T \theta_{t-l'}) \mid \mathcal{F}_{t-l'} \big]\big] \nonumber \\
  &= \mathbb{E}\big[\mathbb{E}\big[\big(\Phi^T(\theta_{t-l'}-\hat{\theta})\big)^T (R_t - D_t\bar{r})+ \big(\Phi^T(\theta_{t-l'}-\hat{\theta})\big)^T(D_t\bar{r}-\tilde{D}^{\mu} \tilde{r})  \mid \mathcal{F}_{t-l'} \big]\big] \nonumber \\  &+\mathbb{E}\big[\mathbb{E}\big[\big(\Phi^T(\theta_{t-l'}-\hat{\theta})\big)^T (\gamma E_t \Phi^T \theta_{t-l'}- \gamma \tilde{D}^{\mu} \tilde{P}^{\mu} \Phi^T \theta_{t-l'}) \mid \mathcal{F}_{t-l'} \big]\big] \nonumber \\
&+\mathbb{E}\big[\mathbb{E}\big[\big(\Phi^T(\theta_{t-l'}-\hat{\theta})\big)^T (\tilde{D}^{\mu} \Phi^T \theta_{t-l'} - D_t \Phi^T \theta_{t-l'}) \mid \mathcal{F}_{t-l'} \big]\big] \nonumber \\
  &\leq 2R\bar{r}(1-\rho)^l + \mathbb{E}\big[\mathbb{E}\big[\big(\Phi^T(\theta_{t-l'}-\hat{\theta})\big)^T(D_t\bar{r}-\tilde{D}^{\mu} \tilde{r}) \mid \mathcal{F}_{t-l'} \big]\big] \nonumber \\
 &+\mathbb{E}\big[\mathbb{E}\big[\big(\Phi^T(\theta_{t-l'}-\hat{\theta})\big)^T (\gamma E_t \Phi^T \theta_{t-l'}- \gamma \tilde{D}^{\mu} \tilde{P}^{\mu} \Phi^T \theta_{t-l'}) \mid \mathcal{F}_{t-l'} \big]\big] \nonumber \\
 &+\mathbb{E}\big[\mathbb{E}\big[\big(\Phi^T(\theta_{t-l'}-\hat{\theta})\big)^T (\tilde{D}^{\mu} \Phi^T \theta_{t-l'} - D_t \Phi^T \theta_{t-l'}) \mid \mathcal{F}_{t-l'} \big]\big], 
\end{align}
where the inequality in \eqref{eq:eta-theta} holds because using a similar argument as in \eqref{eq: reward bound}, for all $\mathcal{G}(H_t) = B$ and $a_t = a$,  we have $r_t - D_t \tilde{r} = \sum_s r(s,a) \pi(s \mid H_t) - \pi(s \mid B) \leq 2\bar{r} (1-\rho)^l$.

Next, let $P^{\mu}_t$ denote the true probability transition matrix of the POMDP, i.e.,
\begin{align*}
P^{\mu}_t(B',a'\mid a,\mathcal{G}(H_t)=B) = \mu(a'\mid B)\sum_{y,s,s'} \mathbb{I}[\mathcal{G}\big(B \,\|\, \{y,a\}\big) = B'] \Phi(y\mid s')\mathcal{P}(s' \mid s,a) \pi(s \mid H_t).  
\end{align*}
We can bound the first term $\mathbb{E}\big[\mathbb{E}\big[\big(\Phi^T(\theta_{t-l'}-\hat{\theta})\big)^T(D_t\bar{r}-\tilde{D}^{\mu} \tilde{r}) \mid \mathcal{F}_{t-l'} \big]\big]$ in \eqref{eq:eta-theta} as follows:
\begin{align}\label{eq:recursive}
    & \mathbb{E}\big[\mathbb{E}\big[\big(\Phi^T(\theta_{t-l'}-\hat{\theta})\big)^T\big(  D_t \tilde{r}  -  \tilde{D}^{\mu} \tilde{r}\big) \mid \mathcal{F}_{t-l'} \big]\big]\nonumber \\
   &=  \mathbb{E}\big[\big(\Phi^T(\theta_{t-l'}-\hat{\theta})\big)^T\big(\mathbb{E}[  D_t \tilde{r}  - \tilde{P}^{\mu} \tilde{D}^{\mu} \tilde{r} \mid \mathcal{F}_{t-l'}]\big) \big]\nonumber \\
   &=   \mathbb{E}\big[\big(\Phi^T(\theta_{t-l'}-\hat{\theta})\big)^T\big(\mathbb{E}[   (D_t \tilde{r}  - \tilde{P}^{\mu}D_{t-1} \tilde{r}) \mid \mathcal{F}_{t-l'} ]+ \mathbb{E}[  (\tilde{P}^{\mu}D_{t-1} \tilde{r} - \tilde{P}^{\mu} \tilde{D}^{\mu}\tilde{r}) \mid \mathcal{F}_{t-l'}]\big)\big] \nonumber \\
   &=   \mathbb{E}\big[\big(\Phi^T(\theta_{t-l'}-\hat{\theta})\big)^T\big(\mathbb{E}\big[\mathbb{E}[   (D_t \tilde{r}  - \tilde{P}^{\mu}D_{t-1} \tilde{r})  \mid\mathcal{F}_{t-1}] \mid \mathcal{F}_{t-l'}\big]+ \mathbb{E}[  (\tilde{P}^{\mu}D_{t-1} \tilde{r} - \tilde{P}^{\mu} \tilde{D}^{\mu}\tilde{r}) \mid \mathcal{F}_{t-l'} ] \big)\big] \nonumber \\
   &\leq \mathbb{E}\big[\| \Phi^T(\theta_{t-l'}-\hat{\theta}) \|_{\infty} \| \mathbb{E}\big[   (P^{\mu}_{t}D_{t-1} \tilde{r}  - \tilde{P}^{\mu}D_{t-1} \tilde{r})   \mid \mathcal{F}_{t-l'}\big]+ \mathbb{E}[  (\tilde{P}^{\mu}D_{t-1} \tilde{r} - \tilde{P}^{\mu} \tilde{D}^{\mu}\tilde{r}) \mid \mathcal{F}_{t-l'} ] \|_1 \big] \nonumber \\
   &\leq 2R \bar{r}(1-\rho)^l + \mathbb{E}[  (\tilde{P}^{\mu}D_{t-1} \tilde{r} - \tilde{P}^{\mu} \tilde{D}^{\mu}\tilde{r}) \mid \mathcal{F}_{t-l'} ] \|_1 \big]\cr
    &\leq 2R \bar{r}(1-\rho)^l + R (1-\rho') \| \mathbb{E}[D_{t-1} \tilde{r} - \tilde{D}^{\mu} \tilde{r} \mid \mathcal{F}_{t-l'}]\|_1,
\end{align}
where the first inequality is derived using the Holder's inequality, and the second inequality holds because
\begin{align}
    &\| \mathbb{E}\big[   (P_{t}^{\mu}D_{t-1} \tilde{r}  - \tilde{P}^{\mu}D_{t-1} \tilde{r})   \mid  \mathcal{F}_{t-l'}\big] \|_1 \nonumber \\
    &\leq \mathbb{E}\big[\|(P_{t}^{\mu}D_{t-1} \tilde{r}  - \tilde{P}^{\mu}D_{t-1} \tilde{r}) \|_1 \mid  \mathcal{F}_{t-l'} \big] \nonumber \\
    &= \sum_{a'}\sum_{B'} |\sum_a \sum_B (P_t^{\mu}(B',a' \mid B,a) - \tilde{P}^{\mu}(B',a' \mid B,a))\cdot\big(\mathbb{I}[(B_{t-1},a_{t-1}) = (B,a)] r(B,a)\big)| \nonumber \\
    &\leq 2 \bar{r} \|P^{\mu}_t-\tilde{P}^{\mu}\|_{TV} \nonumber 
    \end{align}
    \begin{align}
    &\leq \bar{r}\sum_{a'}\sum_{B'} |\mu(a \mid B)\sum_{y,s,s'} \big(\mathbb{I}[\mathcal{G}(B \,\|\, \{y,a\} = B'] \Phi(y\mid s')\big) \cdot \big(\mathcal{P}(s' \mid s,a) (\pi(s \mid B) - \pi(s \mid H_t)) \big) | \nonumber \\
    &=  \bar{r}\sum_{B'} |\sum_{y,s,s'} \big(\mathbb{I}[\mathcal{G}(B \,\|\, \{y,a\} = B'] \Phi(y\mid s')\big) \cdot \big(\mathcal{P}(s' \mid s,a) (\pi(s \mid B) - \pi(s \mid H_t)) \big) | \nonumber \\
    &\leq 2\bar{r}(1-\rho)^l,\nonumber
\end{align}
where the last inequality holds using Lemma \ref{lemma: mixing}.
Therefore, solving Eq. \eqref{eq:recursive} recursively, we get
\begin{align}\nonumber
\mathbb{E}\big[\mathbb{E}\big[\big(\Phi^T(\theta_{t-l'}-\hat{\theta})\big)^T\big(  D_t \tilde{r}  -  \tilde{D}^{\mu} \tilde{r} \mid \mathcal{F}_{t-l'}\big) \big]\big] \leq \bar{r} R \Big(2(1-\rho)^l + 2\frac{(1-\rho)^l}{\rho'} + (1-\rho')^{l'} \Big) \ \forall t. 
\end{align}
Similarly, since $\|\Phi^T \theta_{t-l}\|_{\infty} \leq R$, we can bound the sum of the last two terms in \eqref{eq:eta-theta} as 
\begin{align}\nonumber
&\mathbb{E}\big[\mathbb{E}\big[\big(\Phi^T(\theta_{t-l'}-\hat{\theta})\big)^T (\gamma E_t \Phi^T \theta_{t-l'}- \gamma \tilde{D}^{\mu} \tilde{P}^{\mu} \Phi^T \theta_{t-l'}) \mid \mathcal{F}_{t-l'} \big]\big] \cr
&+\mathbb{E}\big[\mathbb{E}\big[\big(\Phi^T(\theta_{t-l'}-\hat{\theta})\big)^T (\tilde{D}^{\mu} \Phi^T \theta_{t-l'} - D_t \Phi^T \theta_{t-l'}) \mid \mathcal{F}_{t-l'} \big]\big] \cr
    &\leq (1+\gamma(1-\rho')) R\Big((1-\rho')^{l'}+2\frac{(1-\rho)^l}{\rho'}\Big).
\end{align}
Therefore, putting everything together, for a constant stepsize $\epsilon$, we have 
\begin{align}\nonumber
    \mathbb{E}[\mathcal{L}(\theta_{t+1})] &\leq (1-2\epsilon+2\epsilon\gamma) \mathbb{E}[\mathcal{L}(\theta_t)] + \epsilon^2(\bar{r} + 2R)^2 + 2\epsilon R\bar{r}(1-\rho)^l \cr
    &+ 2\epsilon(R\bar{r} + R^2(1+(1-\rho')\gamma))\Big((1-\rho')^{l'} + 2\frac{(1-\rho)^l}{\rho'}\Big) + 2C_2 (\bar{r} + 2R) l'\epsilon^2.
\end{align}
Using the above relation recursively, we have
\begin{align}\nonumber
    \mathbb{E}\|\theta_{\tau+l'}-\hat{\theta} \|_2 &\leq (1-2\epsilon+2\epsilon\gamma)^{\tau} \|\theta_{l'}-\hat{\theta} \|_2 + \Big[\frac{1-(1-2\epsilon+2\epsilon\gamma)^{\tau}}{2\epsilon(1-\gamma)}\Big]\cdot \Big[\epsilon^2(\bar{r} + 2R)^2+4\epsilon R\bar{r}(1-\rho)^l \cr
    &+ 2C_2 (\bar{r} + 2R) l'\epsilon^2+ 2\epsilon(R\bar{r} + R^2(1+(1-\rho')\gamma))\Big((1-\rho')^{l'} + 2\frac{(1-\rho)^l}{\rho'}\Big)\Big].
\end{align}
Therefore,
\begin{align}\nonumber
    \mathbb{E}\| \bar{Q}^{\mu}_{\tau+l'} - \tilde{Q}^{\mu} \|_{\infty} &\leq \|\Phi \hat{\theta} - \tilde{Q}^{\mu} \|_{\infty} + (1-2\epsilon+2\epsilon\gamma)^{\tau} \|\theta_{l'}-\hat{\theta} \|_2 \cr
    &+ \Big[\frac{1-(1-2\epsilon+2\epsilon\gamma)^{\tau}}{2\epsilon(1-\gamma)}\Big] 
    \cdot \Big[\epsilon^2(\bar{r} + 2R)^2 + 4C_2 (\bar{r} + 2R) l'\epsilon^2 + 4\epsilon R\bar{r}(1-\rho)^l  \cr
    &+ 2\epsilon(R\bar{r} + R^2(1+(1-\rho')\gamma))\Big((1-\rho')^{l'} + 2\frac{(1-\rho)^l}{\rho'}\Big)\Big].
\end{align}
Finally, by choosing $\epsilon = \frac{1}{\sqrt{\tau}}$ and $l' = \frac{\log \tau}{2 \log(1-\rho')}$ for sufficiently large $\tau$ such that $2\epsilon(1-\gamma) < 1$, we get
\begin{align}\nonumber
   \mathbb{E}\| \bar{Q}^{\mu}_{\tau+l'} &- \tilde{Q}^{\mu} \|_{\infty} \leq \|\Phi \hat{\theta} - \tilde{Q}^{\mu} \|_{\infty} + (1-\frac{2(1-\gamma)}{\sqrt{\tau}})^{\tau} \|\theta_l-\hat{\theta} \|_2 \cr
    &+ \Big[\frac{1-(1-2(1-\gamma)/\sqrt{\tau})^{\tau}}{(1-\gamma)}\Big] 
    \cdot \Big[\frac{(\bar{r} + 2R)^2}{2\sqrt{\tau}} + \frac{C_2 (\bar{r} + 2R) \log \tau}{ \log(1-\rho')\sqrt{\tau}} +  2R\bar{r}(1-\rho)^l  \cr
    &+ \frac{2}{\rho'}\big(1-\rho\big)^l\big(R\bar{r} + R^2(1+(1-\rho)\gamma)\big) + (1-\rho')^{ \frac{\log \tau}{2 \log(1-\rho')}}\big(R\bar{r} + R^2(1+(1-\rho)\gamma)\big)\Big] \cr
    &\!:= \xi_{\textrm{TD-Error}}
\end{align}

\subsection{Proof of Theorem \ref{thm: fin thm}}

Before we prove the regret bound for our policy iteration algorithm, we first mention an important result from the online learning literature, which will be used to derive the result.

Consider a game between two players consisting of \( M \) rounds. At the beginning of each round \( i \), the environment chooses a loss function \( l_i:\mathcal{A}\to [0,1] \), while the learner selects an action \( a_i \in \mathcal{A} \). After both choices are made, the learner observes the loss \( l_i(a_i) \), while the environment observes \( a_i \). The learner's goal is to minimize its regret with respect to a fixed action \( a^{*}\), which is defined as
\[
\bar{\mathcal{R}}_{M} = \sum_{i=1}^{M} \left( l_i(a_i) - l_i(a^*) \right).
\]
The following lemma provides a high probability bound on \( \bar{\mathcal{R}}_M \). 
\begin{lemma}
\label{lemma: game}
   \cite{Bianchi} For the game mentioned above, assume that at round $i$ the learner chooses an action $a_i=a$ with probability $\mu_i(a) \propto \exp(-\eta \sum_{j=1}^{i-1} l_{j}(a))$, where $\eta = \sqrt{8\log |\mathcal{A}|/M}$. Moreover, let $\delta\in (0,1)$ and $a^{*}\in \mathcal{A}$ be an arbitrary fixed action. Then regardless of how the environment plays, with probability at least $1-\delta$, we have
   \begin{align*}
       \bar{\mathcal{R}}_{M} \leq \sqrt{M \log |\mathcal{A}|/2} + \sqrt{M \log (1/\delta)/2}.
   \end{align*}
\end{lemma}
We will now use this result to prove the regret in Theorem \ref{thm: fin thm}. Using the definition of regret, we have
   \begin{align}\label{eq:regret-proof}
    \mathcal{R}_{T} &= (\tau+l')\sum_{i=1}^{M} \mathbb{E}\Big[V^{\mu^*}(\boldsymbol{\pi}(H_0)) - \tilde{V}^{\mu_i}(\mathcal{G}(H_0))\Big] \cr
    &= (\tau+l')\sum_{i=1}^{M} \mathbb{E}\Big[V^{\mu^*}(\boldsymbol{\pi}(H_0)) - \tilde{V}(\mathcal{G}(H_0))\Big] + (\tau+l')\sum_{i=1}^{M} \mathbb{E}\Big[\tilde{V}(\mathcal{G}(H_0)) - \tilde{V}^{\mu_i}(\mathcal{G}(H_0))\Big]\cr
    &\leq M(\tau+l') \xi^{\textrm{SMDP}}_{\textrm{POMDP}} +(\tau+l')\sum_{i=1}^{M} \mathbb{E}\Big[\tilde{V}(\mathcal{G}(H_0)) - \tilde{V}^{\mu_i}(\mathcal{G}(H_0))\Big]\cr,
\end{align}
where the inequality is obtained using Theorem \ref{thm: main result}. 
Therefore, using the Performance Difference Lemma \cite{kakade2002approximately}, for any $i=1,\ldots,M$, we have
\begin{align}\nonumber
   \mathbb{E}\big[\tilde{V}(\mathcal{G}(H_0)) - \tilde{V}^{\mu_i}(\mathcal{G}(H_0))\big]&= \mathbb{E}\Big[\mathbb{E}_{ a \sim \tilde{\mu}, a' \sim \mu_i}\big[\tilde{Q}^{\mu_i}(\mathcal{G}(H_0),a)-\tilde{Q}^{\mu_i}(\mathcal{G}(H_0),a')\big]\Big] \cr
     &=  \mathbb{E}\big[ \mathbb{E}_{a \sim \tilde{\mu}, a' \sim \mu_i}\big[ \bar{Q}_{\tau+l'}^{\mu_i}(\mathcal{G}(H_0),a) - \bar{Q}_{\tau+l'}^{\mu_i}(\mathcal{G}(H_0),a')\big]\Big] \cr
     &+ \mathbb{E}\Big[\mathbb{E}_{a \sim \tilde{\mu}, a' \sim \mu_i}\big[\tilde{Q}^{\mu_i}(\mathcal{G}(H_0),a) - \bar{Q}_{\tau+l'}^{\mu_i}(\mathcal{G}(H_0),a)\big]\Big]\cr
     &+ \mathbb{E}\Big[\mathbb{E}_{a \sim \tilde{\mu}, a' \sim \mu_i}\big[\bar{Q}_{\tau+l'}^{\mu_i}(\mathcal{G}(H_0),a') - \tilde{Q}^{\mu_i}(\mathcal{G}(H_0),a') \big]\Big]. 
     \end{align}
      Next, using Lemma \ref{Lemma: Approx TD}, we know that $\mathbb{E}[\|\bar{Q}_{\tau+l'}^{\mu_i}- \tilde{Q}^{\mu_i}\|_{\infty}] \leq \xi_{\textrm{TD-error}} \ \forall i$. Therefore, we can write  
     \begin{align}\label{eq:inner-q-product}
     &(\tau+l')\sum_{i=1}^{M}\mathbb{E}\Big[\tilde{V}(\mathcal{G}(H_0)) - \tilde{V}^{\mu_i}(\mathcal{G}(H_0))\Big]  \cr
     &=  (\tau + l')\sum_{i=1}^{M} \mathbb{E}\Big[\mathbb{E}_{a \sim \tilde{\mu}, a' \sim \mu_i}\big[ \bar{Q}_{\tau+l'}^{\mu_i}(\mathcal{G}(H_0),a) - \bar{Q}_{\tau+l'}^{\mu_i}(\mathcal{G}(H_0),a')\big]\Big] +2(\tau+l')M\xi_{\textrm{TD-error}} \cr
     &= (\tau + l')\mathbb{E}\Big[\sum_{i=1}^{M} \Big(\big\langle \mu_i(\cdot \mid \mathcal{G}(H_0)), \bar{Q}^{\mu_i}_{\tau+l'}(\mathcal{G}(H_0),\cdot)\big\rangle - \big\langle\tilde{\mu}(\cdot \mid \mathcal{G}(H_0)), \ \bar{Q}^{\mu_i}_{\tau+l'}(\mathcal{G}(H_0),\cdot)\big\rangle\Big)\Big]\cr
     &\qquad+2M(\tau+l')\xi_{\textrm{TD-error}},
\end{align}
where the last equality follows from the linearity of expectation and expanding the inner expectation. Now, we can apply Lemma \ref{lemma: game} to upper-bound \eqref{eq:inner-q-product}. To this end, we can think of a game between the adversary and a player, which occurs over $M$ rounds. In each round $i$, the adversary chooses the loss function $l_i(\cdot) := \bar{Q}^{\mu_i}_{\tau+l'}(\mathcal{G}(H_0), \cdot)$,
\footnote{The Q-function is upper-bounded by $R$ since $\|\theta\|_2 \leq R$ and $\|\phi(B, a)\| \leq 1$. This only scales the regret bound in Lemma \ref{lemma: game} by a factor of $R$.} and the player selects an action $a_i \in \mathcal{A}$ with probability $\mu_i(a_i|\mathcal{G}(H_0))$, which due to the structure of the policy updates in Algorithm 2 follows the same exponential update rule as in Lemma \ref{lemma: game}. Moreover, since any MDP (and in particular the Superstate MDP) admits a deterministic stationary policy \cite{bertsekas}, it follows that $\tilde{\mu}$ is a deterministic policy that always selects an optimal fixed action $a^*$, i.e., 
\[
\tilde{\mu}(a^*|\mathcal{G}(H_0)) = 1 \quad \text{and} \quad \tilde{\mu}(a|\mathcal{G}(H_0)) = 0 \quad \forall a \neq a^*.
\]
By applying Lemma \ref{lemma: game}, we conclude that the expectation of the inner product in \eqref{eq:inner-q-product} is upper-bounded by 
\[
R\left(M\delta + \sqrt{\frac{M \log |\mathcal{A}|}{2}} + \sqrt{\frac{M \log (1/\delta)}{2}}\right).
\]
Substituting this bound into \eqref{eq:inner-q-product} and combining it with \eqref{eq:regret-proof}, we obtain 

\begin{align}\nonumber 
\mathcal{R}_T &\leq (\tau + l')R\left(M\delta + \sqrt{\frac{M \log |\mathcal{A}|}{2}} + \sqrt{\frac{M \log (1/\delta)}{2}}\right) + M(\tau + l')\left(2\xi_{\textrm{TD-error}} + \xi^{\textrm{SMDP}}_{\textrm{POMDP}}\right) \cr &= (\tau + l')R\left(\sqrt{M} + \sqrt{\frac{M \log |\mathcal{A}|}{2}} + \sqrt{\frac{M \log M}{4}}\right) + T\left(2\xi_{\textrm{TD-error}} + \xi^{\textrm{SMDP}}_{\textrm{POMDP}}\right), \end{align} where in the second step, we have chosen $\delta = 1/\sqrt{M}$ and used the fact that $M(\tau + l') = T$. Finally, by choosing $\tau + l' = \sqrt{T}$ and substituting the values for $\xi_{\textrm{TD-error}}$ and $\xi^{\textrm{SMDP}}_{\textrm{POMDP}}$, we obtain the desired result.

\medskip
\subsection{Improving Bounds in \cite{mahajan} using Lemma \ref{lemma: norm bound}}
We state an important lemma from \cite{mahajan}. For additional details, the reader is referred to \cite{mahajan}. Before stating the lemma, we first present the definition of the approximate information state, as defined in \cite{mahajan}. Note that we have simplified the definition for the case of POMDPs.
\begin{definition}
    Define $z \in \mathbb{Z}$ to be an approximate information state and $\sigma: \mathbb{H} \to \mathbb{Z}$ to be the approximate information state generator. Further, define $\hat{r}: \mathbb{Z} \times \mathcal{A} \to \mathbb{R}$ to be the reward approximation function. Additionally, let $(\epsilon,\delta)$ be fixed constants. Then $Z_t = \sigma(H_t)$ satisfies the following properties:
    \begin{itemize}
        \item For any history $H_t$ and action $a_t$, we have
        \begin{align*}
            |\mathbb{E}[R_t \mid H_t = h_t, A_t = a_t] - \hat{r}(\sigma(H_t),a_t)| \leq \epsilon.
        \end{align*}
        \item For any history $H_t$ and action $a_t$ and any Borel subset $B \in \mathbb{Z}$, define $\mu_t(B) = \hat{P}(Z_{t+1} \in B \mid H_t = h_t, A_t = a_t)$ and $v_t(B) = \mathbb{P}(B \mid \sigma(H_t),a_t).$ Then
        \begin{align*}
            d_{TV}(\mu_t,v_t) \leq \delta.
        \end{align*}
    \end{itemize}
\end{definition}
Note that the grouping operator $\mathcal{G}$ in our work corresponds to the approximate information state generator $\sigma$ in their paper. 
\begin{lemma}
    Consider the approximate information state generator $\sigma$ to be a function which takes a history $H \in \mathbb{H}$ to an approximate information state $z \in \mathbb{Z}$, where for simplicity of our analysis, assume $\mathbb{Z}$ is a discrete set. Define a fixed point equation for the approximate information states as follows: 
    \begin{equation}
        \hat{V}(z,a) = \max_{a}\big[\hat{r}(z,a) + \gamma \sum_{\mathbb{Z}} \hat{V}(z') \hat{P}(z'\mid z,a)\big].\nonumber
    \end{equation}
    Let $\hat{Q}^*$ denote the solution of the fixed point equation and $Q^*$ denote the optimal Q-function for the POMDP.
    Then, for all histories $H_t \in \mathbb{H}$, using results from Lemma 49 and Theorem 27 in \cite{mahajan}, we have
    \begin{equation}
        |\hat{V}^*(H_t) -  V^*(\sigma(H_t))| \leq \frac{\epsilon}{1-\gamma} + \frac{2\gamma\delta\bar{r}}{(1-\gamma)^2}.\nonumber
    \end{equation}
\end{lemma}

In the following, we show how to modify the proof in \cite{mahajan} to obtain a better bound using Lemma \ref{lemma: norm bound}.

\textbf{Proof:}\\
   \begin{align}
       &|\hat{V}^*(H_t) -  V^*(\sigma(H_t))|\nonumber \\
       &= |\max_{a} \big[\mathbb{E}[R_t + \gamma V^*(H_{t+1})\mid H_t = h_t, A_t = a_t] \big] - \max_a \big[\hat{r}(\sigma(h_t),a_t) + \gamma \sum_{z'} \hat{V}(z') \hat{P}(z'\mid z,a)\big]| \nonumber \\
       &\leq |\mathbb{E}[R_t \mid H_t = h_t, A_t = a'] - \hat{r}(\sigma(h_t),a'| + \gamma |\mathbb{E}[V^*(H_{t+1})\mid H_t = h_t, A_t = a'] - \sum_{z'} \hat{V}(z')  \hat{P}(z'\mid z,a')| \nonumber \\
       &\leq \epsilon + \gamma |\sum_{y} V^*(h_t \,\|\, \{y,a'\}) \mathbb{P}(y \mid  H_t = h_t, A_t = a') - \sum_{z'} \hat{V}(z')  \hat{P}(z'\mid \sigma(h_t),a')| \nonumber \\
       &= \epsilon + \gamma |\sum_{z'} \Big(\sum_{y : \sigma(h_t \,\|\, \{y,a'\} = z') } V^*(h_t \,\|\, \{y,a'\}) \frac{\mathbb{P}(y \mid  H_t = h_t, A_t = a')}{\sum_{y : \sigma(h_t \,\|\, \{y,a'\} = z') } \mathbb{P}(y \mid  H_t = h_t, A_t = a')}\Big)  \hat{P}(z'\mid H_t,a') \nonumber \\
       &- \sum_{z'} \hat{V}(z')  \hat{P}(z'\mid \sigma(h_t),'a)|
   \end{align} 
Suppose, for all $H_t$, $|\hat{V}^*(H_t) -  V^*(\sigma(H_t))| \leq \omega.$ This implies that 
\begin{equation}
    \Big|\sum_{y : \sigma(h_t \,\|\, \{y,a'\} = z') } V^*(h_t \,\|\, \{y,a'\}) \frac{\mathbb{P}(y \mid  H_t = h_t, A_t = a')}{\sum_{y : \sigma(h_t \,\|\, \{y,a'\} = z') } \mathbb{P}(y \mid  H_t = h_t, A_t = a')} - \sum_{y : \sigma(h_t \,\|\, \{y,a'\} = z') } \tilde{V}(z') \Big| \leq \omega
\end{equation}

At this point, we can use Lemma \ref{lemma: norm bound}. Therefore, we obtain
\begin{align}
    \omega &\leq \epsilon + \gamma\frac{\delta \bar{r}}{ (1-\gamma)} + \gamma\omega\Big(1-\frac{\delta}{2}\Big) \nonumber \\
    &\leq \frac{\epsilon}{1-\gamma+\delta/4} + \frac{\gamma\delta\bar{r}}{(1-\gamma)(1-\gamma+\delta/2)}.
\end{align}
\subsection{Comparison of our bounds in Theorem \ref{thm: main result} with \cite{semih}}
 We can do a direct comparison with Theorem 4.4 of \cite{semih} and show how we obtained our improved bound. Specifically, our Theorem \ref{thm: fin thm} aggregates the bound from our Theorem \ref{thm: main result}, the TD-learning error from Lemma \ref{Lemma: Approx TD}, and the final error due to policy optimization. Each component can be compared as follows:
\begin{enumerate}
    \item Error due to the mismatch between optimal value functions of the POMDP and SuperState MDP — corresponds to Theorem \ref{thm: main result} and is comparable to $\epsilon_{inf}$ in \cite{semih}.
    \item Standard TD-learning error — the terms in Lemma \ref{Lemma: Approx TD} without $(1 - \rho)^l$, comparable to the first term of $\epsilon_{critic}$ in \cite{semih}.
    \item Additional TD-learning error from sampling mismatch, i.e., the terms in Lemma \ref{Lemma: Approx TD} with $(1-\rho)^l$ (comparable to $\epsilon_{pa}$ of \cite{semih}.
    \item Function approximation error — similar to $l_{CFA}$ in \cite{semih}.
    \item Policy optimization error - comparable to $\epsilon_{actor}$ in \cite{semih}.
\end{enumerate}

We improve upon the bounds in (1) and (3) which leads to the difference mentioned at the beginning of this response. We believe our analysis offers a detailed decomposition and sharp insight into the sources of error.

\section{Analysis of the Practical Validity of Assumption \ref{as:UFSC}}
\subsection{Analyzing Assumption \ref{as:UFSC} for a Practical Example}
\label{toy-example}
We consider a practical example of modelling \textbf{Customer Behavior Modeling in Retail}. 

Here:

States represent engagement levels such as \{Uninterested, Browsing, Considering, Purchasing\}.

Observations are features like the number and type of clicks (e.g., “Viewed Product”, “Added to Cart”).

In such systems, observations from adjacent engagement states tend to overlap significantly: for instance, both “Browsing” and “Considering” may involve product views and occasional cart additions. Furthermore, customer behavior is highly dynamic—users frequently move between these states in short time spans (e.g., from "Considering" back to "Browsing" or forward to "Purchasing"). This overlap in observation distributions and frequent transitions among states leads to high mixing, thereby increasing the Dobrushin coefficient and ensuring filter stability.

Next we present a simplified example to model customer behaviour:

\textbf{States:} $s_0$ (Uninterested), $s_1$ (Browsing), $s_2$ (Considering), $s_3$ (Purchasing).

\textbf{Actions:} $a_0$: Show generic homepage, $a_1$: Recommend trending products

\textbf{Observations:} $y_0$: No clicks, $y_1$: Viewed product, $y_2$: Added to cart, $y_3$: Purchased

\textbf{Transition Probabilities for $a_0$:}

\[
\begin{array}{c|cccc}
s_t \rightarrow s_{t+1} & s_0 & s_1 & s_2 & s_3 \\
\hline
s_0 & 0.4 & 0.4 & 0.1 & 0.1 \\
s_1 & 0.3 & 0.3 & 0.2 & 0.2 \\
s_2 & 0.2 & 0.3 & 0.3 & 0.2 \\
s_3 & 0.1 & 0.2 & 0.4 & 0.3 \\
\end{array}
\]

\textbf{Transition Probabilities for $a_1$:}

\[
\begin{array}{c|cccc}
s_t \rightarrow s_{t+1} & s_0 & s_1 & s_2 & s_3 \\
\hline
s_0 & 0.4 & 0.3 & 0.2 & 0.1 \\
s_1 & 0.2 & 0.4 & 0.2 & 0.2 \\
s_2 & 0.1 & 0.3 & 0.4 & 0.2 \\
s_3 & 0.1 & 0.2 & 0.3 & 0.4 \\
\end{array}
\]

\textbf{Observation Kernel:}

\[
\begin{array}{c|cccc}
s_t \rightarrow y_t & y_0 & y_1 & y_2 & y_3 \\
\hline
s_0 & 0.8 & 0.2 & 0.0 & 0.0 \\
s_1 & 0.3 & 0.5 & 0.2 & 0.0 \\
s_2 & 0.1 & 0.3 & 0.4 & 0.2 \\
s_3 & 0.0 & 0.1 & 0.3 & 0.6 \\
\end{array}
\]

Therefore, referring to Theorem 5 of \cite{kara_dob}, we calculate the Dobrushian coeffients to be $\delta(P) = 0.5$ and $\delta(\phi) = 0.1$, which satisfies the sufficient condition:
\[
(1-\delta(P))(1-\delta(\phi)) < 1,
\]
implying that the Filter Stability condition is satisfied for this case.

We believe that Assumption \ref{as:UFSC} covers many practical examples and future work can consider a multi-step variant where the system exhibits contraction after every $k$ steps. We believe that our results can be extended to such a setting.

\subsection{A Counter Example}
To motivate future work, we also provide a counter example (suggested by one of the reviewers) where Assumption $\ref{as:UFSC}$ does not hold. 

\begin{figure}[h!]
    \centering
    \includegraphics[width=0.3\linewidth]{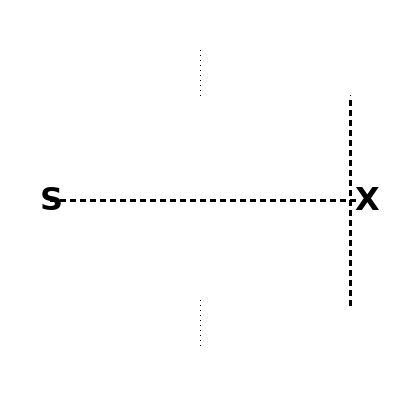} 
    \caption{T-Maze }
    \label{fig:tmaze}
\end{figure}

Consider a T-Maze as described in \cite{bakker}. Here we present an infinite-horizon variant. The domain consists of a T-shaped maze, as shown in Fig \ref{fig:tmaze}. The agent is initially located at the beginning of a corridor (position S), which is followed by a junction (position X) where the agent can step into two different directions, and can keep going into that direction. The goal of the agent is to step into the correct direction, where the correct direction is determined by a piece of information the agent is given at the very beginning, when he is located at the beginning of the corridor (position S). The agent is given a non-zero reward $R > 0$ in every cell after taking the correct direction.

The T-maze is naturally modelled as a POMDP. The hidden states consist of

\begin{itemize}
    \item An initial state $s_0$
    \item States $s_i^d$ for every position in the corridor, including the junction, standing for the fact that the agent is in position $i$
 and the initial observation specified that the correct direction is $d \in \{1,2\}$ (consider that $d=1$ stands for "left" and $d=2$ stands for right).
 \item States $r_i^d$ for $i \in \mathbb{N}$
 and $d \in \{1,2\}$, standing for the fact that the agent is in the $i^{th}$ cell after stepping into the first of the two possible directions, and given that the initial observation said that the correct direction to take is direction $d$
 \item States $q_i^d$ for $i \in \mathbb{N}$
 and $d \in \{1,2\}$, standing for the fact that the agent is in the $i^{th}$ cell after stepping into the second of the two possible directions, and given that the initial observation said that the correct direction to take is direction $d$.

Note that transitions in the hidden state space are deterministic. Given the entire history of observations and actions, the agent can determine the current hidden state, hence belief states are sharp, always assigning probability 1 to a hidden state and 0 to the others. If only a suffix of the history is available, the agent can only determine its current location, and it cannot distinguish between states $s_i^1$
 and $s_i^2$
, between states $r_i^1$
 and $r_i^2$
, and between states $q_i^1$
 and $q_i^2$
. In other words, belief states will assign probability 1/2 to each of the two possible states.

The optimal value function, given an entire history $H$ is roughly $V^* = R(1/(1-\gamma))$. However, the value of any superstate 
 $\mathcal{G}(H)$ is roughly 
$V^*/2$, since the value is determined by the belief, which is uniform over non-distinguishable stats once we remove the first observation. The difference between the two values is $V^*/2$ 
, which is constant wrt to $l$ as opposed to the bound which we obtain in Theorem \ref{thm: main result} which goes to 0 for $l \rightarrow \infty$. The reason why Theorem \ref{thm: main result} does not hold in this case is due to the contraction property required for Assumption \ref{as:UFSC} being non-applicable due to the unique structure of the POMDP.

\end{itemize}
\section{Experimental Results}

\subsection{Effect of History Length and Observation Noise}

Previous works in the literature have focused heavily on solving POMDPs by considering past $k$ observations to design optimal policy. However, for the sake of completeness, we evaluate the performance of our algorithm on a partially observable variant of the FrozenLake-v1 environment. The environment is modified to introduce observation noise, where with probability $p$, the agent receives a random observation different from the true state. All the implementation code is available at this code \href{https://github.com/ameyanjarlekar/Policy-Based-RL-For-POMDPs}{repository}.

The agent uses a history-based state representation by maintaining the last $k$ observation-action pairs. We study the impact of both the observation noise probability $p$ and the history length $k$ on the agent's performance.

\paragraph{Setup.}
\begin{itemize}
    \item Environment: FrozenLake-v1 with \texttt{is\_slippery=False}.
    \item Observation Noise: With probability $p$, the observed state is replaced with a random incorrect state.
    \item History: The agent encodes the last $k$ (observation, action) pairs as the state.
    \item Algorithm: POLITEX with TD(0)-based Q-value estimation and exponentiated gradient policy updates.
    \item Compute: All the experiments were performed on the Google Colab CPU.
\end{itemize}

\paragraph{Results.} The following plot captures the average reward per episode for varying history lengths $k$ and observation noise levels $p$. The moving average was computed over a sliding window of episodes.

\begin{figure}[h!]
    \centering
    \includegraphics[width=0.7\linewidth]{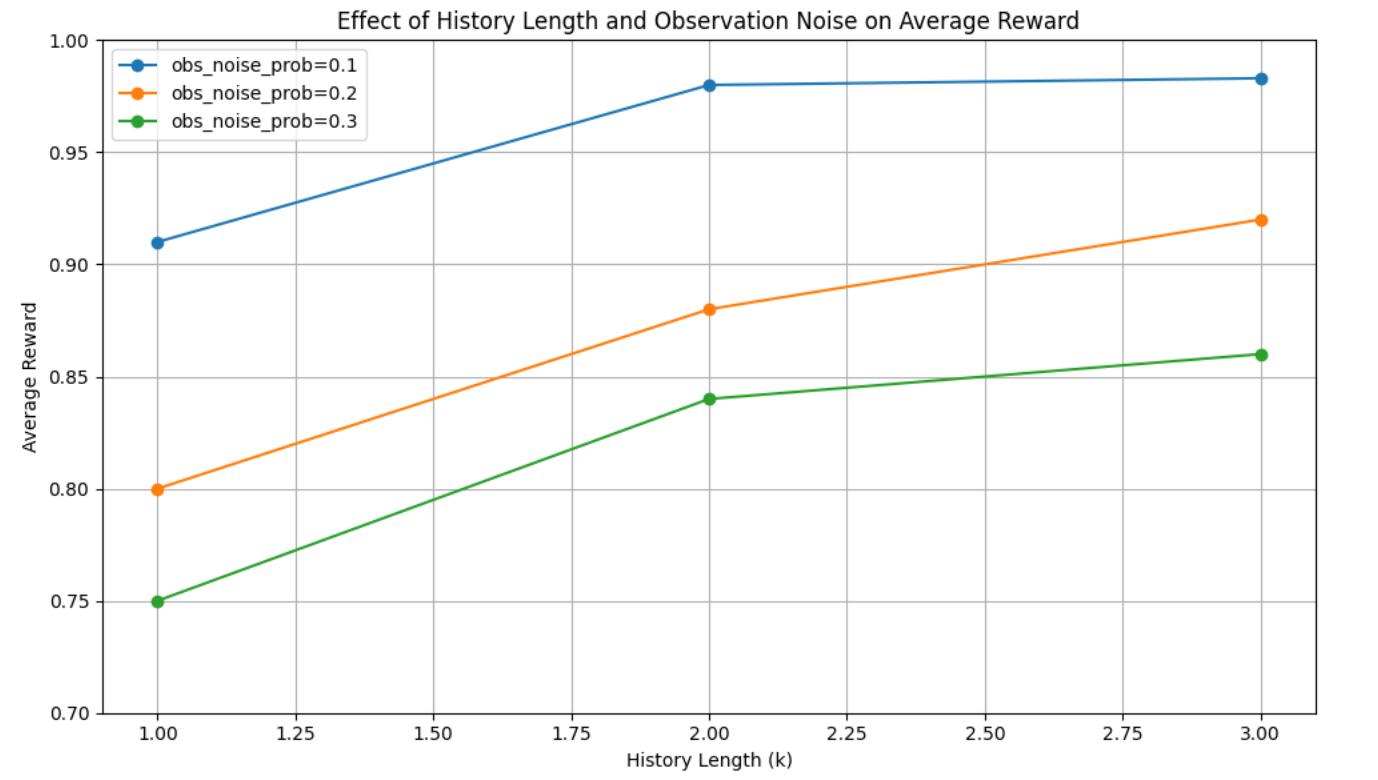} 
    \caption{Average reward per episode for varying history lengths $k$ and observation noise levels $p$. }
    \label{fig:reward_plot}
\end{figure}

\paragraph{Observations.}
\begin{itemize}
    \item Increasing the history length $k$ consistently improves performance under partial observability.
    \item Higher observation noise degrades performance, but the use of longer histories mitigates the effect.
    \item Even with moderate to high noise levels ($p = 0.3$), using $k=3$ allows the agent to recover much of the performance.
\end{itemize}

\subsection{Comparison with \cite{semih}}
 In our paper, we show that we improve upon the bound in \cite{semih}. using a computationally lighter algorithm. To experimentally illustrate this result, we consider an example for a simple partially observable Markov decision process (POMDP) with two states, two actions, and two observations. The environment is stochastic, with state transitions and observations defined by fixed probabilities, and rewards designed to encourage taking the correct action in the hidden state.

To handle partial observability, we represent the agent’s state by a finite history of recent action-observation pairs, with history lengths of 1 and 2 tested.

For each setting, the agent trains over 200 episodes, each of fixed length 20 steps. The learning rate $\alpha$ is set to 0.1 and the discount factor $\gamma$ to 0.9. Policies are represented tabularly and updated greedily with respect to Q-values after each episode.

We measure the agent's performance by total reward accumulated per episode and analyze learning curves as well as final average rewards (averaged over the last 20 episodes) to assess convergence and policy quality.

\begin{figure}[h!]
    \centering
    \includegraphics[width=0.7\linewidth]{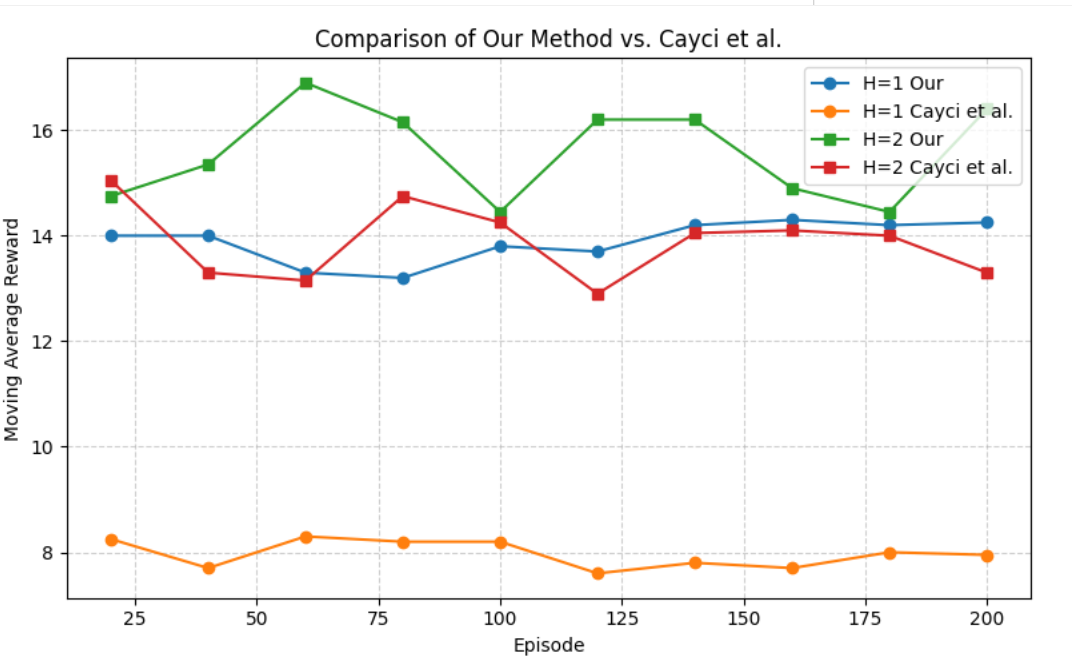} 
    \caption{Comparing Moving Average Reward }
    \label{fig:reward_plot2}
\end{figure}

Additionally, several prior works have empirically demonstrated the effectiveness of using finite histories of observations as surrogates for hidden states in partially observable reinforcement learning problems. For example, \cite{expt1} leverages pretrained language transformers to compress observation histories into compact representations, showing improved sample efficiency on POMDP benchmarks. This is similar to our work if we consider the feature vectors generated by the language transformers as the feature vectors. Similarly, \cite{expt2} presents a framework that adaptively uses privileged state information during training while deploying policies that rely on observation histories. Other related studies, such as \cite{expt3} and \cite{expt4}, support the use of history-based or memory-augmented policies with policy gradient methods.

\end{document}